\documentclass{article}
\usepackage[final,main]{neurips_2026}
\usepackage{microtype}
\usepackage{graphicx}
\usepackage{subcaption}
\usepackage{booktabs} 
\usepackage{amsmath,amssymb,amsthm}
\newtheorem{theorem}{Theorem}
\newtheorem{lemma}{Lemma}
\newtheorem{proposition}{Proposition}
\newtheorem{assumption}{Assumption}
\newtheorem{remark}{Remark}
\usepackage{hyperref}

\usepackage{amsmath}
\usepackage{amssymb}
\usepackage{mathtools}
\usepackage{amsthm}
\usepackage{enumitem}
\usepackage{booktabs}
\usepackage{graphicx}
\usepackage{multirow} 
\usepackage{makecell}
\usepackage[dvipsnames]{xcolor}
\usepackage{color, colortbl}
\usepackage{placeins}
\usepackage{pifont}
\usepackage{subcaption}
\usepackage{algpseudocode}

\usepackage[utf8]{inputenc} 
\usepackage[T1]{fontenc}    
\usepackage{hyperref}       
\usepackage{url}            
\usepackage{booktabs}       
\usepackage{amsfonts}       
\usepackage{nicefrac}       
\usepackage{microtype}      
\usepackage{xcolor}         
\usepackage{graphicx}
\usepackage{tabularx}
\usepackage{graphicx}
\usepackage{bm}
\usepackage{caption}
\title{DEG: Diffusion-Enhanced GFlowNet for\\ Solving Vehicle Routing Problems}
\author{
Ni Zhang$^1$, 
Zhiqin Zhang$^{1}$, 
Ling Pan$^2$,
Hoong Chuin Lau$^1$,
Zhiguang Cao$^1$\\
$^{1}$Singapore Management University\\
$^{2}$ Hong Kong University of Science and Technology\\
\texttt{ni.zhang.2025@phdcs.smu.edu.sg, zhiqinzhang@smu.edu.sg, lingpan@ust.hk}\\
\texttt{hclau@smu.edu.sg, zgcao@smu.edu.sg}
}
\begin{document}



\maketitle
\begin{abstract}
Traditional neural solvers for solving vehicle routing problems (VRPs) often suffer from limited solution diversity, motivating the development of Generative Flow Network (GFlowNet)–based models. However, the effectiveness of these models is frequently constrained by insufficient flow expansion in high-reward regions, limiting their ability to distribute probability across promising solution routes. The deeper exploration could yield superior results. Diffusion models, in contrast, provide stronger structural guidance for exploration. These two paradigms are naturally complementary: GFlowNet can supply edge-level signals for diffusion to embed, while diffusion can guide broader exploration. Leveraging this synergy, we propose Diffusion-Enhanced GFlowNet (DEG), a novel framework that integrates GFlowNet with a diffusion model to encourage richer flow expansion toward high-reward regions and derive higher-quality solutions. Specifically, DEG exploits GFlowNet’s inherent diversity to generate edge-specific backward signals, applies the stochastic noise schedule of diffusion to perturb these signals, and then denoises them within the GFlowNet paradigm. To further improve scalability, we introduce a specialized decoder capable of dynamically adapting to diverse problem scales. Extensive experimental evaluations on synthetic and real-world datasets, including instances with up to 10,000 nodes, demonstrate that DEG consistently achieves favorable performance compared to baseline methods. The code is available at \href{https://github.com/ZHANG-NI/DEG}{https://github.com/ZHANG-NI/DEG}.
\end{abstract}

\section{Introduction}

The Vehicle Routing Problem (VRP) is a fundamental NP-hard combinatorial optimization problem that generalizes the Traveling Salesman Problem (TSP) by incorporating practical constraints such as vehicle capacity limitations and route feasibility requirements. The objective of VRPs is to minimize overall travel costs by determining optimal routes for a fleet of vehicles dispatched from a central depot to serve multiple customers while satisfying these constraints. VRPs are central to a wide range of applications, including logistics and distribution planning \citep{bochtis2009vehicle,bochtis2010vehicle,cattaruzza2017vehicle,hsiao2018last,konstantakopoulos2022vehicle,leng2021distribution,zhang2026afl}, agricultural field operations \citep{wu2023research,yao2022green,utamima2022agricultural}, and urban delivery systems \citep{carrabs2017exact,gayialis2019vehicle,cirovic2014green}. Owing to their importance, VRPs have attracted extensive research attention. Traditional metaheuristics \citep{braysy2001genetic,baker2003genetic,mohammed2017solving,osman1993metastrategy,chiang1996simulated,kuo2010using,rizzoli2007ant,montemanni2005ant,bell2004ant} have demonstrated effectiveness in solving VRPs. However, they are often computationally expensive and heavily dependent on carefully designed solution components and extensive hyperparameter tuning tailored to specific instances. These limitations have spurred growing interest in deep learning–based neural solvers, which aim to provide data-driven heuristics that improve inference efficiency while reducing reliance on expert knowledge.

Recent neural solvers for VRPs can be broadly categorized into Autoregressive (AR) and Non-Autoregressive (NAR) models. AR models, such as those based on pointer networks and Transformer architectures \citep{kwon2020pomo,luo2023neural,bresson2021transformer,huang2025rethinking,yi2026radar}, have achieved strong solution quality. However, their sequential decoding process introduces significant computational bottlenecks, limiting scalability to large instances. NAR models provide a promising alternative by generating complete solutions in a single forward pass, greatly improving inference efficiency. Early NAR methods, however, struggled with solution quality, motivating the development of Generative Flow Network (GFlowNet) \citep{bengio2021flow,bengio2023gflownet}. While GFlowNet-based solvers \citep{zhangadversarial,kimant,zhang2026hybrid} leverage sampling diversity to improve solution quality, they often converge to sub-optimal route subsets and fail to fully explore high-reward regions, leading to insufficient flow expansion into these more promising areas and hindering the discovery of higher-quality solutions.


Diffusion models~\citep{sohl2015deep, ho2020denoising} recently emerged as a powerful approach in the VRP domain for their strong exploration capability~\citep{sun2023difusco,li2023t2t,li2024fast, Wang2025}. Yet existing diffusion-based methods remain largely confined to the TSP and depend on supervised frameworks to generate edge-level noise signals. While GFlowNet and diffusion models each have distinct strengths and limitations, their integration for VRP solving has not been explored. These two paradigms are naturally complementary: diffusion can guide GFlowNet toward more effective flow expansion in high-reward regions, and GFlowNet can provide edge-level signals and extend the applicability of diffusion methods.

Based on this, we propose Diffusion-Enhanced GFlowNet (DEG), a novel framework that integrates GFlowNet with diffusion models for solving VRPs. DEG leverages the intrinsic diversity of GFlowNet to generate edge-specific backward signals, which serve as a noise basis for the diffusion process. Unlike conventional trajectory-level backward signals that collapse into a single aggregated value, edge-specific signals provide richer structure for diffusion noise modeling. DEG perturbs these signals with a stochastic noise schedule and denoises them within the GFlowNet paradigm, preventing premature convergence to sub-optimal policies and promoting broader exploration of high-reward regions. This mechanism enhances forward probability allocation across promising routes and enables more effective flow expansion, enabling the discovery of higher-quality solutions. To further enhance generalization across varying problem scales, we develop the Graph-Scale Adapter (GSA), a specialized module that aligns spatial statistics across different graph scales and significantly improves scalability. Experimental evaluations demonstrate DEG's strong performance and scalability on both TSP and Capacitated Vehicle Routing Problem (CVRP), achieving competitive results compared to existing baselines and exhibiting robustness on instances containing up to 10,000 nodes. Our contributions are summarized as follows:
\begin{itemize}[leftmargin=*]
    \item We propose Diffusion-Enhanced GFlowNet (DEG), which leverages intrinsic diversity of GFlowNet to generate edge-specific backward signals that serve as the adaptive noise basis for integrating diffusion, then applies the diffusion's noise schedule to inject noise into these signals, and denoises them through the GFlowNet paradigm, enabling deeper flow expansion toward high-reward regions to derive better solutions.
    \item We design the Graph-Scale Adapter (GSA) to align spatial statistics across varying instance sizes, which enhances the model's ability to generalize to large-scale problems.
    \item We conduct comprehensive experiments on TSP and CVRP, and results show that DEG consistently achieves favorable results on both synthetic and real-world datasets.
\end{itemize}

\section{Related Work}
The deep learning based neural solvers for VRPs can be broadly classified into Autoregressive (AR) and Non-Autoregressive (NAR) methods, respectively.

\textbf{AR methods} sequentially construct solutions step-by-step, conditioned on previously selected nodes. Transformer-based architectures have been particularly influential due to their ability to effectively capture long-range dependencies. A notable early example is the Attention Model (AM) \citep{koolattention}, which applies self-attention mechanisms to enable end-to-end routing policy learning. To further enhance solution diversity and robustness, POMO \citep{kwon2020pomo} builds on AM by employing multi-start policy optimization, exploring multiple potential optimal paths simultaneously. Building on the strengths of autoregressive architectures, LEHD \citep{luo2023neural} introduces a light-encoder, heavy-decoder design, achieving substantial performance gains across standard VRP benchmarks.

\textbf{NAR methods}, by contrast, generate complete solutions in a single forward pass, drastically reducing inference time and enabling parallel solution generation. This characteristic makes them particularly appealing for large-scale applications. Diffusion-based models represent a recent and promising class of NAR approaches. DIFUSCO \citep{sun2023difusco} is a pioneering diffusion-based model tailored for the TSP variant of VRPs. Building upon this foundation, T2T \citep{li2023t2t}, Fast T2T \citep{li2024fast} and DEITSP \citep{Wang2025} significantly improve solution quality while retaining inference efficiency. However, all these diffusion-based methods rely on supervised learning with optimal solutions and are mainly limited to TSP. GFlowNet-based solvers form another mainstream approach. GFACS~\citep{kimant} employs GFlowNet to enhance the ACO heuristic. AGFN~\citep{zhangadversarial} adopts an adversarial strategy to generate routes in an end-to-end manner. Different from those methods, in this work, we propose Diffusion-Enhanced GFlowNet (DEG), which leverages diffusion to enhance the exploration of high-reward regions, thereby improving the solution quality for VRPs.

\section{Preliminary}
\textbf{Problems Definition.}
In both the TSP and CVRP, the instance is defined on a complete graph \( \mathcal{G} = (\mathcal{V}, \mathcal{E}) \), where \( \mathcal{V} = \{n_0, n_1, \ldots, n_v\} \) denotes the set of nodes and \( \mathcal{E} \) the set of edges. The cost \( d(n_i, n_j) \) associated with each edge is typically defined as the Euclidean distance between nodes \( n_i \) and \( n_j \). In the TSP, the objective is to find a minimum-cost Hamiltonian circuit that visits all nodes exactly once. In the CVRP, \( n_0 \) represents the depot, and the remaining nodes correspond to customers with non-negative demands \( d_i \). The goal is to construct a set of depot-rooted routes \( \tau = \{q_1, \ldots, q_N\} \) such that each customer is visited exactly once, the total demand on each route does not exceed the vehicle capacity \( C \), and the overall travel cost is minimized.

\textbf{Diffusion Models.} Diffusion models have recently emerged as powerful generative approaches, achieving favorable results in domains such as image synthesis \citep{sohl2015deep, ho2020denoising, rombach2022high,dhariwal2021diffusion}, audio generation \citep{huang2023make,kongdiffwave}, and drug design \citep{schneuing2024structure,xie2022polymeric}. Diffusion models operate by corrupting clean inputs with noise and denoise them to reconstruct the original data. Typically, given a clean input $x_0$, the noise process perturbs it by adding Gaussian noise to obtain a noisy sample $x_t$ at time step $t$, following 
\begin{equation}
q(x_t \mid x_0) = \mathcal{N}(x_t; \sqrt{\alpha_t} x_0, (1 - \alpha_t) \mathbf{I}),
\end{equation}
where $\alpha_t \in (0, 1)$ controls the noise level. The model can capture the structure of the target distribution by extracting meaningful signals from noisy states and recovering the original data. 

\textbf{GFlowNet.} Generative Flow Network (GFlowNet) is a variational inference algorithm that treat sampling from a target probability distribution as a sequential decision-making process~\citep{bengio2021flow, bengio2023gflownet}, and is widely applied to structured generation tasks, including molecule design~\citep{nica2022evaluating, jain2022biological,zhu2023sample}, neural architecture search~\citep{soin2024generative,ikram2024evolution}, and combinatorial optimization~\citep{zhang2023let,zhangrobust,chenorder}. The objective of GFlowNet is to learn a policy that samples objects in proportion to a given reward distribution. From a high-level perspective, a GFlowNet consists of two complementary components: forward sampling policy and backward policy. The forward pass includes the source flow $Z(\mathcal{G}; \theta)$ and a forward policy $P_F(\tau; \theta)$ that governs trajectory construction from initial to terminal states. The backward pass includes the reward function $R(\tau)$ and a backward policy $P_B(\tau)$ that defines reverse transitions. Together, these bidirectional passes form the basis for training via trajectory balance objectives~\citep{malkin2022trajectory}, defined as:
\begin{equation}
\mathcal{L}_\text{TB}(\tau; \theta) = \left( \log \frac{Z(\mathcal{G}; \theta) P_{F}({\tau}; \bm{\theta})}{R(\tau) P_{B}({\tau})} \right)^2.
\label{eq:tb_loss equation}
\end{equation}

\section{Diffusion-Enhanced GFlowNet}
\label{sec:GFlowNet-Enhanced Diffusion (GED)}
We introduce Diffusion-Enhanced GFlowNet (DEG), a framework designed to overcome a core limitation of existing GFlowNet-based VRP solvers, namely, their inability to effectively expand flow toward high-reward regions. DEG integrates GFlowNet with diffusion by leveraging edge-specific backward signals, which serve as a structured noise basis for the diffusion process. A stochastic noise schedule perturbs these signals, which are subsequently denoised within the GFlowNet paradigm, enabling richer exploration and more sufficient flow expansion. To further strengthen scalability and inference quality across varying problem sizes, DEG incorporates a Graph-Scale Adapter (GSA) decoder, which dynamically aligns spatial statistics across different graph scales.

\subsection{Edge-specific Backward Signals and Noise Schedule}
\begin{figure*}[!htb]
    \centering
    \includegraphics[width=\textwidth, trim=0cm 10cm 1cm 0cm, clip]{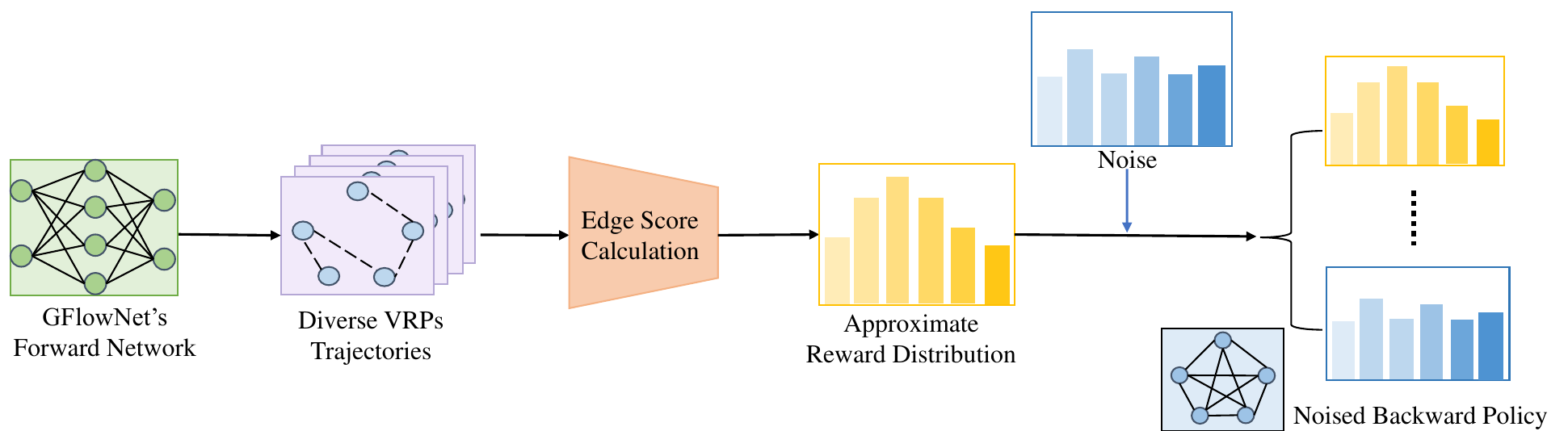}
    \caption{Illustration of Edge-specific Backward Signals Generation and Noise Schedule.}
    \label{fig: noise}
\end{figure*}
\label{sec:Reward-level Feedback and Noise Schedule}
Figure \ref{fig: noise} illustrates the generation of edge-specific backward signals and noise schedule. We first use GFlowNet to generate a diverse set of trajectories and assign standardized scores to each edge, estimating an approximate reward distribution that serves as edge-specific backward signals. Although the individual trajectories produced during GFlowNet sampling may not be optimal, their collective information forms a robust reward distribution that guides the diffusion learning process. A noise schedule then injects controlled noise into this distribution to derive a noised backward policy.

By leveraging the inherent diversity of GFlowNet, the model can generate a wide range of trajectories, after which the edge score calculation is applied. The edge score calculation begins by normalizing the reward of each trajectory \( \tau_k \in \mathcal{T} \) using the following formulation:
\begin{equation}
    -\log {R}({\tau}_{k}) = D({\tau}_{k}) - \frac{1}{m} \sum_{l=1}^{m} D({\tau}_{l}),
\label{eq:reward}
\end{equation}
where \( D(\tau_k) \) is the total distance of trajectory \( \tau_k \), defined as: $D(\tau_k) = \sum_{l=1}^{N-1} d(q_l, q_{l+1}).$
Here, $d(q_l, q_{l+1})$ denotes the Euclidean distance between nodes $q_l$ and $q_{l+1}$. A higher reward \( R(\tau_k) \) corresponds to a shorter, higher-quality trajectory. We define the reward of an edge $e$ as the average reward of the trajectories in which it appears:
\begin{equation}
    R(e) = \frac{1}{\left| \left\{ \tau_k \;\middle|\; \tau_k \in \mathcal{T},\; e \in \tau_k \right\} \right|}
    \sum_{\tau_k \in \mathcal{T},\, e \in \tau_k} R(\tau_k),
\end{equation}
where the denominator counts the number of trajectories containing edge $e$. This edge-level reward $R(e)$ reflects the global significance of edge $e$ across the sampled solution space. To standardize $R(e)$ across different instances, we use z-score normalization \citep{montgomery2010applied}: $\hat{R}(e) = \frac{R(e) - \mu_{R}}{\sigma_{R}},$
where \( \mu_{R} \) and \( \sigma_{R} \) denote the mean and standard deviation of edge rewards within the instance. To convert the standardized scores into bounded confidence values, we apply the cumulative distribution function (CDF) \citep{ross2014introduction} of the standard normal distribution:
\begin{equation}
    x_0 = \Phi\left( \hat{R}(e) \right),
\label{eq:cdf}
\end{equation}
where \( \Phi(\cdot) \) denotes the CDF of the standard normal distribution. The value $x_0$ represents an approximate reward distribution that provides edge-specific backward signals and serves as the basis for adding noise to integrate diffusion into GFlowNet. The noise schedule then injects Gaussian noise at each training step to obtain a noised backward policy $x_t$:
\begin{equation}
\label{eq:backflow}
    x_t = \sqrt{{\alpha}_t} \cdot x_0 + \sqrt{1 - {\alpha}_t} \cdot \epsilon, \quad \epsilon \sim \mathcal{N}(0, 1),
\end{equation}
where ${\alpha}_t = \prod_{k=1}^{t}(1 - \beta_k)$, and $\beta_k \in (0, 1)$ is the corruption ratio at diffusion step $k$. To ensure full transformation of the clean signal into noise by the final training iteration $I$, $\beta_k$ should be defined such that $\prod_{k=1}^{I}(1 - \beta_k) \approx 0$. We use continuous Gaussian noise to match the inherently continuous nature of the backward policy in GFlowNet.

\textbf{Effect of Edge-specific Backward Signals and Noise Schedule on High-Reward Region Exploration}. DEG uses a noise schedule that progressively corrupts the edge-specific backward signals during training, together with the Trajectory Balance (TB) objective~\citep{malkin2022trajectory}. We characterize the learning target induced by this schedule at each noise level, with proofs provided in Appendix~\ref{app:target_transition}. In practice, the noised signal is clamped as
$\tilde{x}_t[e]=\max(\mathrm{clamp}(x_t[e],0,1),\delta)\in[\delta,1]$ with $\delta=10^{-10}$, ensuring that the backward-side trajectory weight
$W_t(\tau)=\prod_{e\in\tau}\tilde{x}_t[e]$, which replaces $P_B(\tau)$ in $\mathcal{L}_\text{TB}$ (Eq.~\ref{eq:tb_loss equation}), remains strictly positive.

Because TB minimizes a squared log-ratio, the injected noise affects the optimum through $\mathbb{E}_{\epsilon}[\log W_t(\tau)]$. For a fixed graph $\mathcal{G}$ and signal $x_0$, the population minimizer at noise level $t$ is
\begin{equation}
\label{eq: pf_r_pb}
P_F^{t}(\tau) \propto R(\tau)e^{\mu_t(\tau)},
\qquad
\mu_t(\tau)=\sum_{e\in\tau}\mathbb{E}_{\epsilon}\big[\log \tilde{x}_t[e]\big],
\end{equation}
which holds exactly at every noise level (Proposition~\ref{prop:exact_target}). The target can therefore be viewed as reward-proportional sampling modified by an edge-specific factor, whose dependence on $x_0$ is gradually reduced by the noise schedule. At the beginning of training, no noise is added ($\bar{\alpha}_t=1$), giving
\begin{equation}
\label{eq:train_low}
P_F^{0}(\tau) \propto R(\tau)\prod_{e\in\tau}x_0[e].
\end{equation}
Compared with $\pi_R\propto R$, this distribution gives more weight to trajectories containing high-scoring edges. Its expected reward is higher than that of $\pi_R$ whenever $\prod_{e\in\tau}x_0[e]$ is positively correlated with $R(\tau)$ (Proposition~\ref{prop:low_noise}), which motivates constructing $x_0$ from edge-reward statistics. This early stage therefore favors exploitation. Only when $\prod_{e\in\tau}x_0[e]\propto R(\tau)$ does Eq.~\ref{eq:train_low} reduce to $P_F\propto R(\tau)^2$; our CDF-based construction of $x_0$ does not assume this special case.

As the noise increases, the influence of $x_0$ weakens. At full corruption ($\bar{\alpha}_t=0$), $\tilde{x}_t[e]$ is independent of $x_0$, and $\mu_t(\tau)=c_T|\tau|$ for a constant $c_T$. For TSP, all tours contain exactly $n$ edges, so this term is identical across trajectories and
\begin{equation}
\label{eq:train_high}
P_F(\tau) \propto R(\tau).
\end{equation}
Before reaching this endpoint, $\mathrm{KL}(P_F^{t}\,\|\,\pi_R)$ admits a bound that decreases to zero as $\bar{\alpha}_t\to0$ (Theorem~\ref{thm:kl_bound}), providing a useful characterization of the late training stage. For CVRP, a trajectory with $K(\tau)$ routes contains $n+K(\tau)$ edges. The endpoint target therefore becomes
$P_F(\tau)\propto R(\tau)e^{c_TK(\tau)}$ with $c_T<0$. It is reward-proportional among solutions with the same number of routes, while additionally favoring solutions with fewer routes (Remark~\ref{rem:cvrp}).

Overall, the noise schedule gradually removes the edge-specific factor from the learning target, moving from the edge-guided distribution in Eq.~\ref{eq:train_low} toward the reward-proportional distribution in Eq.~\ref{eq:train_high}. This gives a natural transition from early exploitation to broader exploration.

\subsection{Diffusion-enhanced GFlowNet Framework}
In our framework, the model learns to construct solutions by drawing on both the noisy backward policy and the reward within the GFlowNet paradigm, which denoises the noised backward policy during training by extracting its informative weights, and using them to enhance forward-policy learning. The loss over a set of trajectories $\mathcal{T} = \{\tau_1, \tau_2, \ldots, \tau_m\}$, based on the denoised forward policy which estimates reward distribution, is defined as:
\begin{equation}
\mathcal{L}(\mathcal{T}; \theta) = \sum_{k=1}^{m} \left( \log \frac{Z(\mathcal{G}; \theta) P_{F}({\tau}_{k}; \bm{\theta})}{R(\tau_k) P_{B}({\tau}_{k})} \right)^2.
\label{eq:loss equation}
\end{equation}
Here, source flow $Z(\mathcal{G}; \theta)$ is a learnable partition function derived from Eq.~\ref{eq:Z_calculation}, and reward $R(\tau_k)$ is defined in Eq.~\ref{eq:reward}. The forward probability $P_{F}({\tau}_{k}; \bm{\theta})$ is computed as:
\begin{equation}
\label{eq:forward}
P_{F}({\tau}_{k}; \bm{\theta}) = \prod_{t=1}^{N} P_{F}(s_t|s_{t-1}; \bm{\theta}),
\end{equation}
where $s_t = \{q_1, q_2, \ldots, q_t\}$ denotes the state at step $t$ during the construction (or destruction) of the trajectory. The term $P_{F}(s_t|s_{t-1})$ is obtained from the forward policy $\eta(\mathcal{G}^*, \bm{\theta})$, which is derived from the GFlowNet forward network described in Sec.~\ref{sec:forward network}. 

Moreover, $\eta(\mathcal{G}^*, \bm{\theta})$ can be viewed as a \emph{heatmap} over all edges of the current graph: it is a $|V|\times |V|$ probability matrix where each entry $\eta(i,j)$ gives the likelihood of selecting the edge $(v_i,v_j)$ from the current node $v_i$ to a candidate node $v_j$. During decoding, the next node of the route is sampled according to these edge probabilities. After each selection, the graph state $\mathcal G^*$ is updated by marking the chosen node as visited and updating feasibility masks (e.g., vehicle capacity), and the process is repeated until all customers are visited and a complete route is
constructed.

Similarly, the backward transition probability $P_B({\tau}_{k})$ in Eq.~\ref{eq:loss equation} is computed as:
\begin{equation}
\label{eq:backward}
P_{B}({\tau}_{k}; \bm{\theta}) = \prod_{k=1}^N P_{B}(s_{k-1}|s_k; \bm{\theta}),
\end{equation}
where $P_{B}(s_{t-1}\mid s_t;\bm{\theta})$ is derived from the noised backward policy $x_t$. Specifically, for a partial route state $s_t=\{q_1,q_2,\ldots,q_t\}$, this corresponds to removing the last node $q_t$ from the current partial route, which eliminates the edge $(q_t,q_{t-1})$. Thus $P_{B}(s_{t-1}\mid s_t;\bm{\theta})$ is taken directly from the value of $x_t$
at the position $(q_t,q_{t-1})$.
\subsection{Graph-Scale Adapter for Inference}
\label{sec:Decoder Strategies for GED}
During training, we adopt GFlowNet’s inherent sampling strategy to generate diverse trajectories to provide edge-specific signal. This enables efficient sampling and strengthens training effectiveness, with a dedicated model trained for a specific problem size. However, this strategy does not generalize well at inference time due to scale mismatches between training and test instances. 

Specifically, the DEG forward policy learns an approximate reward distribution guided by feedback from the noisy backward flow. However, this design may increase the risk of overfitting to spatial patterns specific to the training scale, such as average neighbor distances and node observation radii defined by the sparsity factor $w$. For example, when trained on 100-node instances, the model implicitly learns to solve instances with an average neighbor distance of approximately $\frac{1}{\sqrt{100}}$ and a node observation distance of $\frac{1}{\sqrt{w}}$. When evaluated on a 500-node instance, however, the average neighbor distance drops to $\frac{1}{\sqrt{500}}$, while the observation distance becomes $\frac{1}{\sqrt{5w}}$. This spatial inconsistency hampers generalization and often leads to suboptimal inference. The best outcomes of large-scale instances appear at earlier training stages, when the model has not yet overfitted and its estimation of spatial structures remains effective and less biased. We empirically validate this behavior, with details provided in Appendix~\ref{sec:GSA Comparison Results}.

To mitigate this issue, we propose the \textit{Graph-Scale Adapter} (GSA), which aligns the spatial patterns of test instances to those of fixed-size training instances. GSA introduces a scaling factor $\alpha$ that adjusts node distances and the sparsity factor in the test graph to match the spatial characteristics of the training graph. The scaling factor is defined as: $\alpha \propto \sqrt{\frac{|\mathcal{V}_{\text{test}}|}{|\mathcal{V}_{\text{train}}|}}.$
This adjustment ensures structural consistency across different scales. For instance, if the model is trained on 100-node graphs and evaluated on 500-node graphs, applying GSA ensures that the average neighbor distances and node observation radius remain consistent, approximately $\frac{1}{\sqrt{100}}$ and $\frac{1}{\sqrt{w}}$. To assess the effectiveness of GSA, we apply it during inference on instances of varying sizes. As shown in Appendix~\ref{sec:GSA Comparison Results}, GSA effectively alleviates spatial mismatches caused by scale-specific overfitting and significantly improves the model’s ability to generalize across a broad range of instance sizes.
\section{Experiments}
\begin{table*}[t]
\centering
\vspace{-5mm}
\caption{Comparison of performance and runtime on TSP with varying node sizes. Obj. denotes the total route length; Gap (\%) is computed relative to LKH; and Time (s) indicates the average time required to solve a single instance. For all metrics, lower values indicate better results.}
\resizebox{\textwidth}{!}{
\begin{tabular}{l|ccc|ccc|ccc}
\toprule\toprule
\multirow{2}{*}{\textbf{Method}} 
& \multicolumn{3}{c|}{\textbf{500}} 
& \multicolumn{3}{c|}{\textbf{1000}} 
& \multicolumn{3}{c}{\textbf{3000}} \\
& \textbf{Obj.} & \textbf{Time(s)} & \textbf{Gap(\%)} 
& \textbf{Obj.} & \textbf{Time(s)} & \textbf{Gap(\%)} 
& \textbf{Obj.} & \textbf{Time(s)} & \textbf{Gap(\%)} \\
\midrule
LKH & 16.55 & 31.80 & -- & 23.14 & 97.62 & -- & 40.31 & 776.23 & --\\
GFACS (AISTATS'25) & 22.97 & 8.43 & 38.79 & 38.58 & 18.90 & 66.72 & 69.63 & 81.15 & 72.74\\
AGFN (ICLR'25)  & 19.04 & 0.25 & 15.05 & 27.10 & 0.68 & 17.11 & 55.37 & 2.91 & 37.36\\
DIFUSCO (NeurIPS'23)&19.12&3.52&15.17&27.70&13.60&19.71& 50.39 & 125.02 & 25.01\\
T2T (NeurIPS'23)  &18.91 & 3.18 &14.26 &27.31&12.95&18.02& 56.99 & 115.71 & 41.38\\
Fast T2T (NeurIPS'24)&18.76&0.37&13.35&27.25&1.54&17.76& 57.73 & 4.76 & 43.22\\
DEITSP (KDD'25)&18.64&0.42&12.63&27.11&1.68&17.16& 54.37 & 5.35 & 34.81\\
DEG&\textbf{18.47}&0.28&\textbf{11.60}&\textbf{26.46}&0.72&\textbf{14.35}& \textbf{47.82} & 3.06 & \textbf{18.63}\\
\midrule

\midrule
\multirow{2}{*}{\textbf{Method}} 
& \multicolumn{3}{c|}{\textbf{5000}} 
& \multicolumn{3}{c|}{\textbf{8000}} 
& \multicolumn{3}{c}{\textbf{10000}} \\
& \textbf{Obj.} & \textbf{Time(s)} & \textbf{Gap(\%)} 
& \textbf{Obj.} & \textbf{Time(s)} & \textbf{Gap(\%)} 
& \textbf{Obj.} & \textbf{Time(s)} & \textbf{Gap(\%)} \\
\midrule
LKH  & 51.84 & 1608.11 & --&65.89 &4123.37 &-- & 75.78 & 5350.91 & -- \\
GFACS (AISTATS'25)  & 97.93 & 143.19& 88.91&120.76&275.22&83.28  & 140.84 & 326.04 & 85.85 \\
AGFN (ICLR'25)  & 80.93 & 6.65 & 56.12&118.53&13.09& 79.89 & 143.85 & 14.07 & 89.83 \\
DIFUSCO (NeurIPS'23)  & -- & -- & -- & -- & -- & --& -- & -- & --  \\
T2T (NeurIPS'23)  & -- & -- & -- & -- & -- & --& -- & -- & --  \\
Fast T2T (NeurIPS'24) & -- & -- & -- & -- & -- & --& -- & -- & --  \\
DEITSP (KDD'25)  & -- & -- & -- & -- & -- & --& -- & -- & --  \\
DEG  & \textbf{63.64} & 6.82 & \textbf{22.57} &\textbf{83.13}&14.97&\textbf{26.16} & \textbf{94.67} & 16.47 & \textbf{24.94} \\
\bottomrule\bottomrule
\end{tabular}
}
\begin{flushleft}
\small\textit{``--'' indicates that the model was unable to produce a result for a given instance either within the 90-minute time limit or due to GPU memory constraints.}
\end{flushleft}
\vspace{-4mm}
\label{tab:tsp_split}
\end{table*}
\label{sec:Experiments}
We present experimental results to evaluate the performance of DEG. We first assess its effectiveness on both synthetic instances for TSP and CVRP following previous works~\citep{kimant,zhangadversarial}. Since our model integrates GFlowNet with diffusion, we primarily compare it against GFlowNet-based and diffusion-based baselines. We then evaluate the model's generalization ability on public benchmarks, including TSPLib \citep{arnold2019efficiently} and CVRPLib \citep{reinelt1991tsplib}. Finally, we conduct ablation studies to examine the contribution of individual components in DEG. We additionally conduct more experiments on multiple scale training, GSA applicability, diversity test, learning efficiency and applicability to other COPs, the details of which are presented in Appendix~\ref{sec:Multiple Scale training} to~\ref{sec:Application to Other Combinatorial Optimization Problems}

Following prior work on VRP \citep{sun2023difusco, kimant, zhangadversarial, zhangdiffusion, Wang2025}, we use synthetic datasets for training both TSP and CVRP models. Each TSP instance consists of nodes represented only by their 2D coordinates, sampled independently from the unit square $[0,1]^2$, without any depot or demand. The objective is to compute the shortest tour that visits each node exactly once and returns to the start. For CVRP, each instance includes a single depot and multiple customer nodes served by a vehicle with a fixed capacity $C = 50$. Node locations (both depot and customers) are uniformly sampled from $[0,1]^2$. Each customer is assigned a demand independently drawn from a uniform distribution $U[a, b]$, where $a = 1$ and $b = 9$.
For training, DEG is trained on 100-node instances for both TSP and CVRP. For evaluation, we generate synthetic test sets containing 128 instances per problem size, ranging from 500 to 10,000 nodes (500, 1,000, 3,000, 5,000, 8,000 and 10,000).

Following standard practice \citep{sun2023difusco, Wang2025, graikos2022diffusion}, we adopt a linear noise schedule with $\beta_1 = 10^{-4}$ and $\beta_I = 0.02$. The model is trained for $I = 6000$ steps, using a batch size of 10 instances per step. We use the AdamW optimizer \citep{loshchilovdecoupled} with an initial learning rate of $10^{-4}$, decayed to $10^{-5}$ via a CosineAnnealingLR schedule \citep{loshchilov2022sgdr}. All experiments are conducted on a system with an NVIDIA H100 NVL GPU and an AMD EPYC 9554 CPU.

\subsection{Comparison Study on TSP}
\label{sec:Comparison Study on TSP}
As prior diffusion-based work has primarily focused on the TSP \citep{sun2023difusco, li2023t2t, li2024fast, Wang2025} in the VRP domain, we first evaluate DEG on synthetic TSP instances. Besides the heuristic solver LKH \citep{helsgaun2000effective}, we mainly compare it against both GFlowNet-based solvers including GFACS \citep{kimant} and AGFN \citep{zhangadversarial}, and diffusion-based solvers including DIFUSCO \citep{sun2023difusco}, T2T \citep{li2023t2t}, Fast T2T \citep{li2024fast}, and DEITSP \citep{Wang2025}. All models are trained on 100-node instances. To ensure fair comparison, and due to space constraints, Table~\ref{tab:tsp_split} reports the performance of models \textit{without} 2-opt post-processing. Following \citep{sun2023difusco, li2023t2t,li2024fast, Wang2025}, results with 2-opt post-processing are provided in Appendix~\ref{Dec:Two-opt} and Table \ref{tab:tsp_two-opt} for completeness. We also provide the comparison on small-scale instances in Appendix~\ref{sec:TSP and CVRP comparison in small-scale problems} and Table~\ref{tab:smallscale_sidebyside}.

As shown in Table~\ref{tab:tsp_split}, DEG outperforms all GFlowNet-based and diffusion-based solvers in terms of solution quality. Specially, it achieves improvements of up to $67.43\%$, $34.19\%$, $5.10\%$, $16.10\%$, $17.17\%$, and $12.05\%$ over GFACS, AGFN, DIFUSCO, T2T, Fast T2T, and DEITSP, respectively. Notably, the performance gap between our model and other baselines becomes more significant as the instance size increases. 
\subsection{Comparison Study on CVRP}
\label{sec:Comparison Study on CVRP}
\begin{table*}[t]
\centering
\caption{Comparison of performance and runtime on CVRP with varying node sizes. Obj. denotes the total route length; Gap (\%) is computed relative to LKH; and Time (s) indicates the average time required to solve a single instance. For all metrics, lower values indicate better results.}
\resizebox{\textwidth}{!}{
\begin{tabular}{l|ccc|ccc|ccc}
\midrule\midrule
\multirow{2}{*}{\textbf{Method}} 
& \multicolumn{3}{c|}{\textbf{500}} 
& \multicolumn{3}{c|}{\textbf{1000}} 
& \multicolumn{3}{c}{\textbf{3000}} \\
& \textbf{Obj.} & \textbf{Time(s)} & \textbf{Gap(\%)} 
& \textbf{Obj.} & \textbf{Time(s)} & \textbf{Gap(\%)} 
& \textbf{Obj.} & \textbf{Time(s)} & \textbf{Gap(\%)} \\
\midrule
LKH & 63.14 & 134.40 & -- & 119.52 & 638.40 & -- & 380.26* & 671.68  & --\\
GFACS (AISTATS’25) & 80.77 & 10.43 & 27.92 & 158.97 & 22.07 & 33.01 & 393.27 & 130.81 & 3.42\\
AGFN (ICLR'25) & 71.05 & 0.43 & 12.53 & 133.97 & 0.75 & 12.09 & 368.97 & 3.84 & -2.97\\
LEHD (NeurIPS’23) & \textbf{64.68} & 0.52 & \textbf{2.44} & \textbf{122.85} & 10.15 & \textbf{2.79}& \textbf{355.95} & 33.75 & \textbf{-6.39}\\
POMO (NeurIPS’20)&80.21&0.50&27.04&164.17&0.92&37.36 & 1094.48 & 7.01 & 187.82\\
POMO(*8) (NeurIPS’20)&75.45&0.67&19.50&135.74&3.13&12.74 & 704.99 & 18.26 & 85.40\\
DEG&69.76&0.45&10.48&129.67&0.80&8.49& 362.11 & 4.08 & -4.77\\
\midrule

\midrule
\multirow{2}{*}{\textbf{Method}} 
& \multicolumn{3}{c|}{\textbf{5000}} 
& \multicolumn{3}{c|}{\textbf{8000}} 
& \multicolumn{3}{c}{\textbf{10000}} \\
& \textbf{Obj.} & \textbf{Time(s)} & \textbf{Gap(\%)} 
& \textbf{Obj.} & \textbf{Time(s)} & \textbf{Gap(\%)} 
& \textbf{Obj.} & \textbf{Time(s)} & \textbf{Gap(\%)} \\
\midrule
LKH(1000) & 642.95* & 2263.34& -- & -- & -- & --& -- & -- & --  \\
GFACS (AISTATS’25)  & 650.62 & 262.63 & 1.19 &1038.30&529.04&--& 1439.19 & 607.26 & -- \\
AGFN (ICLR'25)  & 608.58 & 7.42 & -5.35 &917.24 &15.24 &-- & 1140.46 & 18.58 & -- \\
LEHD (NeurIPS’23)  & 603.89 & 96.75 & -6.08 &938.06&336.70&--& 1213.13 & 556.88 & -- \\
POMO (NeurIPS’20)& 1438.97 & 19.32 & 123.81 & -- & -- & --& -- & -- & -- \\
POMO(*8) (NeurIPS’20) & 1381.63 & 41.86 & 114.89 & -- & -- & --& -- & -- & -- \\
DEG  & \textbf{596.43} & 10.22 & \textbf{-7.24} &\textbf{907.65}&20.32&\textbf{--}& \textbf{1125.04} & 24.95 & \textbf{--} \\
\bottomrule\bottomrule
\end{tabular}
}
\begin{flushleft}
\small\textit{``--'' indicates that the model was unable to produce a result for a given instance either within the 90-minute time limit or due to GPU memory constraints. * denotes LKH run with 1,000 iterations due to time limits.}
\end{flushleft}
\label{tab:CVRP1}
\end{table*}
We present results on the CVRP instances, focusing primarily on comparisons with GFlowNet-based solvers, as no existing diffusion-based model is capable of solving CVRP. Besides the heuristic solver LKH \citep{helsgaun2000effective}, we compare DEG against GFlowNet-based solvers including GFACS \citep{kimant} and AGFN \citep{zhangadversarial}, as well as other neural solvers such as LEHD \citep{luo2023neural} and POMO \citep{kwon2020pomo} to ensure comprehensive evaluation. For LKH, we use a setting of 1000 iterations on 3000-, 5000-, 8000- and 10000-node instances, as larger-scale problems require substantial computational time. Also, all neural models are trained on 100-node synthetic instances, and the routes are evaluated directly without post-processing. Experimental results on small-scale instances are further provided in Appendix~\ref{sec:TSP and CVRP comparison in small-scale problems}, and the corresponding comparisons are summarized in Table~\ref{tab:smallscale_sidebyside}.

As shown in Table~\ref{tab:CVRP1}, DEG achieves better results than GFACS by $13.63\%$, $18.43\%$, $7.92\%$, $8.33\%$, $12.58\%$ and $21.83\%$ on the 500-, 1000-, 3000-, 5000-, 8000- and 10000-node instances, and outperforms AGFN by $1.82\%$, $3.21\%$, $1.86\%$, $2.00\%$, $1.05\%$ and $1.35\%$ on the same instance sizes. When compared to Transformer-based AR solvers, NAR methods often struggle to match their performance without the aid of improvement techniques, partly due to inherent limitations in network depth and complexity. Without relying on additional post-processing techniques, our model DEG is able to deliver competitive performance on large-scale instances, while also offering notable advantages in inference efficiency. Compared to  state-of-the-art AR mode, LEHD, which adopts a heavy decoder transformer network, our model reduces the objective value by $1.24\%$, $3.28\%$ and $7.26\%$ on the 5000-, 8000- and 10000-node instances, respectively, and reduces inference time by $13.46\%$, $92.61\%$, $87.91\%$, $87.37\%$, $93.96\%$ and $95.58\%$ on 500-, 1000-, 3000-, 8000- and 5000-node instances. Furthermore, our model consistently outperforms POMO in objective value and achieves shorter inference time across all instance sizes. Overall, our model achieves competitive results against GFlowNet-based solvers and other neural models, and offers significantly shorter inference time, particularly on large-scale instances. Notably, DEG achieves better performance than both LEHD and POMO on real-world datasets, as detailed in Sec.~\ref{sec:Inference on Real-World Dataset} and Table~\ref{tab:CVRPLib}.

\subsection{Inference on Real-World Dataset}
\begin{table*}[t]
\centering
\caption{Comparison of performance on TSPLib. Gap (\%) is computed relative to existing best-known solutions, and lower values indicate better results. }

\resizebox{0.9\textwidth}{!}{
\begin{tabular}{l|c|c|c|c|c|c|c|c}
\toprule
\multirow{1}{*}{\textbf{TSP}} & \multicolumn{1}{c|}{\textbf{GFACS}} & \multicolumn{1}{c|}{\textbf{AGFN}} & \multicolumn{1}{c|}{\textbf{DEG}}& \multicolumn{1}{c|}{\textbf{DIFUSCO*}}& \multicolumn{1}{c|}{\textbf{T2T*}}& \multicolumn{1}{c|}{\textbf{Fast T2T*}}& \multicolumn{1}{c|}{\textbf{DEITSP*}}& \multicolumn{1}{c}{\textbf{DEG*}} \\

\midrule
50-200 & 29.60 & 19.41 & \textbf{14.37}&0.96&0.37 & 0.26 &0.81& \textbf{0.21} \\
200-1000 & 51.67 & 25.51 & \textbf{17.68}&3.45&1.84 & 1.26 &1.46& \textbf{1.03} \\
\bottomrule
\end{tabular}

}
\begin{flushleft}
\small\textit{``--'' indicates that the model was unable to produce a result for a given instance either within the 90-minute time limit or due to GPU memory constraints. * represents using two-opt after trajectories are generated.}
\end{flushleft}

\label{tab:TSPLib}
\end{table*}

\begin{table*}[t]
\centering
\caption{Comparison of performance on CVRPLib. Gap (\%) is computed relative to existing best-known solutions, and lower values indicate better results.}
\resizebox{0.9\textwidth}{!}{
\begin{tabular}{l|cccccccccc}
\toprule
\textbf{Gap(\%)} & \textbf{L1 (3k)} & \textbf{L2 (4k)} & \textbf{A1 (6k)} & \textbf{A2 (7k)} & \textbf{G1 (10k)} & \textbf{G2 (11k)} & \textbf{B1 (15k)} & \textbf{B2 (16k)} \\
\midrule
GFACS  & 120.21 & 1800.04& 114.37 & 238.19 & 278.11 & 2242.09& 348.54 & 2634.74\\
AGFN   & 1144.20& 27.88  & 29.64  & 24.58  & 142.65 & 27.21  & 20.63  & 160.24 \\
LEHD   & \textbf{14.04}  & 26.30  & 18.90  & 26.40  & 27.23  & 38.45  & 35.94  & 40.76  \\
POMO   & 75.30  & 78.16  & 112.27 & 159.22 & --     & --     & --     & --     \\
{DEG}    & 19.64  & \textbf{26.08}  & \textbf{14.95}  & \textbf{23.80}  & \textbf{23.18}  & \textbf{24.41}  & \textbf{17.04}  & \textbf{23.47}  \\
\bottomrule
\end{tabular}}
\begin{flushleft}
\end{flushleft}
\label{tab:CVRPLib}
\end{table*}

\begin{table*}[t]
\centering
\caption{Ablation Study on TSP and CVRP. Obj. denotes the total route length, and Gap (\%) is computed relative to LKH; lower values indicate better results.}

\resizebox{1.0\textwidth}{!}{
\begin{tabular}{l|cc|cc|cc|cc|cc|cc}
\toprule
\multirow{2}{*}{\textbf{TSP}} & \multicolumn{2}{c|}{\textbf{500}} & \multicolumn{2}{c|}{\textbf{1000}}& \multicolumn{2}{c|}{\textbf{3000}}& \multicolumn{2}{c|}{\textbf{5000}}& \multicolumn{2}{c|}{\textbf{8000}} & \multicolumn{2}{c}{\textbf{10000}} \\
& \textbf{Obj.} & \textbf{Gap(\%)} & \textbf{Obj.}  & \textbf{Gap(\%)} & \textbf{Obj.}  & \textbf{Gap(\%)}& \textbf{Obj.}  & \textbf{Gap(\%)}& \textbf{Obj.}  & \textbf{Gap(\%)}& \textbf{Obj.}  & \textbf{Gap(\%)} \\
\midrule
Backbone & 21.96 & 32.69 & 35.41 & 53.03&64.82&60.80&91.66&76.81&--&--&--&--\\
~~~~+ Diffusion & 18.54 & 12.02 & 26.74 & 15.56&48.45&20.19&64.97&25.33&--&--&--&--\\
~~~~~~~~+ GSA & \textbf{18.47} & \textbf{11.60} & \textbf{26.46} & \textbf{14.35}&\textbf{47.82}&\textbf{18.63}&\textbf{63.64}&\textbf{22.57}&\textbf{83.13} &\textbf{26.16} &\textbf{94.67}&\textbf{24.94}\\
\midrule

\midrule
\multirow{2}{*}{\textbf{CVRP}} & \multicolumn{2}{c|}{\textbf{500}} & \multicolumn{2}{c|}{\textbf{1000}}& \multicolumn{2}{c|}{\textbf{3000}}& \multicolumn{2}{c|}{\textbf{5000}}& \multicolumn{2}{c|}{\textbf{8000}} &\multicolumn{2}{c}{\textbf{10000}} \\
& \textbf{Obj.} & \textbf{Gap(\%)} & \textbf{Obj.}  & \textbf{Gap(\%)} & \textbf{Obj.}  & \textbf{Gap(\%)} & \textbf{Obj.}  & \textbf{Gap(\%)}& \textbf{Obj.}  & \textbf{Gap(\%)}& \textbf{Obj.}  & \textbf{Gap(\%)}\\
\midrule
Backbone & 76.88 & 21.76 & 148.61 & 24.34&389.18 &2.35&--&--&--&--&--&--\\
~~~~+ Diffusion & 70.86 & 12.23 & 132.65 & 10.74 &368.99& -2.96&--&--&--&--&--&--\\
~~~~~~~~+ GSA & \textbf{69.76} & \textbf{10.48} & \textbf{129.67} & \textbf{8.49}& \textbf{362.11} & \textbf{-4.77} & \textbf{596.43} & -- & \textbf{907.65} & --& \textbf{1125.04} & -- \\
\bottomrule
\end{tabular}
}
\begin{flushleft}
\small\textit{``--'' indicates that the model was unable to produce a result for a given instance either within the 90-minute time limit or due to GPU memory constraints.}
\end{flushleft}
\label{tab:ablation}
\end{table*}

\label{sec:Inference on Real-World Dataset}
To assess the generalization ability of our model, we evaluate it on real-world benchmark datasets, TSPLib \citep{arnold2019efficiently} and CVRPLib \citep{reinelt1991tsplib}. 

Following prior works \citep{sun2023difusco, li2023t2t, li2024fast, Wang2025}, we evaluate our model on TSPLib with test instances ranging from 50 to 200 nodes and from 200 to 1000 nodes, respectively. We adopt the same experimental setup and baseline methods as described in Section~\ref{sec:Comparison Study on TSP}. The results are reported in Table~\ref{tab:TSPLib} in the Appendix. Our model outperforms GFACS and AGFN in gap by $51.45\%$ and $25.97\%$ on 200-500 node instances, and $65.78\%$ and $30.69\%$ on 500-1000 node instances, respectively. Compared to the diffusion-based models with two-opt post-processing, DEG achieves superior gap performance over DIFUSCO, DEITSP, T2T, and Fast T2T by $88.59\%$, $84.56\%$, $72.73\%$ and $74.07\%$ on 200-500 node instances, and $70.72\%$, $45.11\%$, $19.84\%$ and $30.82\%$ on 500-1000 node instances. Overall, our model outperforms both GFlowNet-based models and diffusion-based solvers on TSPLib, across both 50-200 and 200-1000 instances.

We test CVRPLib containing instances with the number of node ranging from 3000 to 16000. The experiment setting and baselines are same with comparison on CVRP benchmarks. As shown in Table~\ref{tab:CVRPLib}, DEG achieves consistently lower optimality gaps than all baseline methods across a wide range of real-world CVRPLib instances. Compared to GFACS, DEG reduces the gap on the same set of instances. Compared to AGFN, our model reduces significant gap on the same set of instances. Relative to LEHD, DEG achieves gap improvements of $0.84\%$ on L2, $20.89\%$ on A1, $9.85\%$ on A2, $14.87\%$ on G1, $36.51\%$ on G2, $52.59\%$ on B1, and $42.42\%$ on B2, while LEHD performs slightly better on L1. Finally, compared to POMO, our model significantly reduces the gap by $73.92\%$, $66.63\%$, $86.68\%$, and $85.05\%$ on L1, L2, A1, and A2. These results demonstrate the superior scalability and generalization ability of DEG in solving large-scale real-world CVRP instances.

\subsection{Ablation Study}
\label{sec:Ablation Study}
\textbf{Component Evaluation.} We evaluate the contribution of each key component in DEG through an ablation study. We begin by training the backbone model, which involves learning a pure GFlowNet to generate trajectories, serving as the baseline for subsequent comparison. Next, we integrate the diffusion mechanism into the backbone by incorporating the reward-level feedback and noise schedule described in Sec.\ref{sec:Reward-level Feedback and Noise Schedule}, in order to evaluate the effectiveness of the diffusion component. Finally, the GSA decoder (Sec.~\ref{sec:Decoder Strategies for GED}) is applied on top of the previous model to assess its impact. 

The results are presented in Table~\ref{tab:ablation}. Incorporating diffusion into the backbone leads to significant performance improvements. Adding the GSA module on top of the diffusion-enhanced model further improves performance on the TSP and CVRP across the same instance sizes. In addition to improving performance, GSA also reduces computational resource consumption by shortening the observation distance and decreasing the number of connections per node. As a result, the model is able to efficiently solve 10000-node TSP instances as well as 5000-, 8000- and 10000-node CVRP instances.

\section{Conclusion}
In this paper, we introduce Diffusion-Enhanced GFlowNet (DEG), a generative framework that integrates GFlowNet with diffusion to overcome key limitations of existing GFlowNet-based solvers for vehicle routing problems (VRPs). Specifically, their insufficient flow expansion in high-reward regions. To enable this integration, DEG leverages the intrinsic diversity of GFlowNet to provide edge-specific backward signals as a noise basis for diffusion integration. It then employs the diffusion's noise schedule to inject noise into these signals, and denoises them through the GFlowNet paradigm, guiding sufficient flow expansion toward high-reward regions to derive better solutions. Furthermore, to enhance scalability and generalization ability, we incorporate a Graph-Scale Adapter (GSA) that enables effective adaptation across diverse instance sizes. Extensive experiments on synthetic benchmarks and real-world datasets, including instances with up to 10,000 nodes, show that DEG delivers favorable performance compared to diffusion-, GFlowNet-based and other neural baselines. These results demonstrate that DEG has strong scalability, efficient inference, and robust generalization for solving VRPs.

While DEG outperforms the strong AR model like LEHD in inference time and generalization, one of the \textbf{limitations} is that DEG's parallel decoding still falls short in small-scale settings, where AR models excel at fine-grained path construction. Bridging this gap remains an important direction for our future work, which involves enhancing the local decision-making capacity of NAR models. In addition, extending DEG to a wider range of VRP variants is another venue we plan to explore.
\clearpage
\bibliographystyle{plain}
\bibliography{example_paper} 

\newpage
\appendix
\onecolumn
\section{More Methodology Detail}

\subsection{Forward Network and Training Paradigm}
\label{sec:forward network}
Following prior work \citep{kimant, zhangadversarial, sun2023difusco, xin2021neurolkh}, we adopt a Graph Neural Network (GNN) as the forward network $\theta$, incorporating distance information as edge features. GNN employs a sparsification strategy on the input graph $\mathcal{G}$ to reduce computational complexity. Specifically, the graph retains only the $w$ shortest edges per node, resulting in a sparse graph $\mathcal{G}^* = (\mathcal{V}, \mathcal{E}^*)$. 
The node coordinates $\mathbf{u}_v \in \mathbb{R}^2$ and edge distances $\mathbf{u}_e \in \mathbb{R}$ are first embedded into initial node and edge features $\mathbf{h}^0_i, \mathbf{e}^0_{ij} \in \mathbb{R}^d$ by learnable linear layers, where $d$ represents the number of hidden units per layer. Then, $\mathbf{h}^0_i$ and $\mathbf{e}^0_{ij}$ are processed by GNN to capture the underlying relationships between nodes and edges. At each layer $l$, node and edge features are updated as:

\begin{align}
    \mathbf{h}^{l+1}_i &= \mathbf{h}^l_i + \mathrm{ACT}(\mathrm{BN}(\mathbf{U}^l \mathbf{h}^l_i + \mathcal{A}_{j \in \mathcal{N}_i}(\sigma(\mathbf{e}^l_{ij}) \odot \mathbf{V}^l \mathbf{h}^l_j))), \\
    \mathbf{e}^{l+1}_{ij} &= \mathbf{e}^l_{ij} + \mathrm{ACT}(\mathrm{BN}(\mathbf{P}^l \mathbf{e}^l_{ij} + \mathbf{Q}^l \mathbf{h}^l_i + \mathbf{R}^l \mathbf{h}^l_j)).
\end{align}

Here, \( \mathbf{U}^l, \mathbf{V}^l, \mathbf{P}^l, \mathbf{Q}^l, \mathbf{R}^l \in \mathbb{R}^{d \times d} \) are learnable parameters of the GNN layers; SiLU \citep{elfwing2018sigmoid} is used as the activation function for all intermediate layers; BN denotes batch normalization; \( \sigma \) is the sigmoid function \citep{han1995influence}; \( \odot \) is the Hadamard product; and \( \mathcal{A} \) denotes mean pooling over neighbors.

The denoised forward policy $\eta(\mathcal{G}^*, \bm{\theta})$ is computed using a multi-layer perceptron (MLP) applied to the final-layer GNN embeddings, followed by a sigmoid activation to normalize outputs. The source flow $Z(\mathcal{G}; \theta)$ is computed from the final-layer edge embedding $e^d$ as follows:
\begin{equation}
Z(\mathcal{G}; \theta) = M \cdot \text{ReLU}(J \cdot e^d + b) + f,
\label{eq:Z_calculation}
\end{equation}
where $M$, $J$, $b$, and $f$ are learnable parameters, and ReLU is the activation function.

\section{Analysis of the Learning Target under the Noise Schedule}
\label{app:target_transition}

We describe the implemented signal construction, then derive the population target at each noise level, analyze the two endpoints, and bound its distance from reward-proportional sampling.

\subsection{Setup and well-posedness}
\label{app:setup}

Fix an instance graph $\mathcal{G}$. A trajectory $\tau$ is a complete construction sequence with reward $R(\tau)>0$, and $|\tau|$ denotes its number of edges. For TSP, $|\tau|=n$ for every tour. For CVRP, $|\tau|=n+K(\tau)$, where $K(\tau)$ is the number of routes.

\textbf{Edge signal.}
The edge reward $\bar{R}(e)$ is the average reward of sampled trajectories containing edge $e$. After z-score normalization, the clean signal is
$x_0[e]=\Phi(\hat{R}(e))\in(0,1)$
(Eq.~\ref{eq:cdf}).

\textbf{Corruption.}
At noise level $t$ (Eq.~\ref{eq:backflow}),
\begin{equation}
x_t[e]
=
\sqrt{\bar{\alpha}_t}\,x_0[e]
+
\sqrt{1-\bar{\alpha}_t}\,\epsilon[e],
\qquad
\epsilon[e]\sim\mathcal{N}(0,1),
\qquad
\bar{\alpha}_t=\prod_{k=1}^{t}(1-\beta_k),
\end{equation}
where the noise is sampled independently across edges.

\textbf{Clamping.}
Because $x_t[e]$ may fall outside $[0,1]$, we use
\begin{equation}
\tilde{x}_t[e]
=
\max\big(\mathrm{clamp}(x_t[e],0,1),\delta\big)
\in[\delta,1],
\qquad
\delta=10^{-10}.
\end{equation}

\textbf{Backward-side trajectory weight.}
The backward term used in TB is
\begin{equation}
\label{eq:app_wt}
W_t(\tau)
=
\prod_{e\in\tau}\tilde{x}_t[e]
\in[\delta^{|\tau|},1].
\end{equation}
The factors $\tilde{x}_t[e]$ are used directly rather than normalized over parent states. Thus, $W_t$ is a positive trajectory weight rather than a normalized backward transition kernel. The clamping guarantees $W_t>0$.

\textbf{Objective.}
The TB loss for a trajectory is
\begin{equation}
\label{eq:app_loss}
L_t(\tau;\theta)
=
\Big(
\log Z(\mathcal{G};\theta)
+
\log P_F(\tau;\theta)
-
\log R(\tau)
-
\log W_t(\tau)
\Big)^2.
\end{equation}
Since $W_t>0$ and $\log W_t$ is bounded, the loss is well defined. Equivalently, this is TB with the trajectory reward
$\tilde{R}_t(\tau)=R(\tau)W_t(\tau)$.

\begin{assumption}[Population setting]
\label{ass:population}
The signal $x_0$ is fixed, trajectories are sampled from a distribution with full support, and optimization is over all strictly positive trajectory distributions $P_F$ together with the normalizer $Z>0$.
\end{assumption}

\subsection{Exact target at every noise level}
\label{app:exact_target}

\begin{lemma}[Bias--variance decomposition]
\label{lem:bias_variance}
For every trajectory $\tau$ and noise level $t$,
\begin{equation}
\mathbb{E}_{\epsilon}\big[L_t(\tau;\theta)\big]
=
\Big(
\log Z(\mathcal{G};\theta)
+
\log P_F(\tau;\theta)
-
\log R(\tau)
-
\mu_t(\tau)
\Big)^2
+
\sigma_t^2(\tau),
\end{equation}
where
\begin{equation}
\label{eq:app_mu}
\mu_t(\tau)
=
\mathbb{E}_{\epsilon}\big[\log W_t(\tau)\big]
=
\sum_{e\in\tau}h_t\big(x_0[e]\big),
\end{equation}
with
\begin{equation}
h_t(u)
=
\mathbb{E}_{\epsilon\sim\mathcal{N}(0,1)}
\left[
\log
\max\left(
\mathrm{clamp}
\left(
\sqrt{\bar{\alpha}_t}\,u
+
\sqrt{1-\bar{\alpha}_t}\,\epsilon,
0,1
\right),
\delta
\right)
\right],
\end{equation}
and
$\sigma_t^2(\tau)=
\mathrm{Var}_{\epsilon}[\log W_t(\tau)]$,
which does not depend on $\theta$.
\end{lemma}

\begin{proof}
Using
$\mathbb{E}[(a-X)^2]
=
(a-\mathbb{E}[X])^2+\mathrm{Var}(X)$,
set
\[
a=\log Z+\log P_F(\tau),
\qquad
X=\log R(\tau)+\log W_t(\tau).
\]
For fixed $\tau$, $R(\tau)$ is deterministic, while
$\log W_t(\tau)\in[|\tau|\log\delta,0]$.
The result follows from
$\log W_t(\tau)=\sum_{e\in\tau}\log\tilde{x}_t[e]$
and linearity of expectation.
\end{proof}

\begin{proposition}[Exact target]
\label{prop:exact_target}
Under Assumption~\ref{ass:population}, the expected TB loss at noise level $t$ has the unique population minimizer
\begin{equation}
\label{eq:app_target}
P_F^{t}(\tau)
=
\frac{R(\tau)e^{\mu_t(\tau)}}{Z_t},
\qquad
Z_t
=
\sum_{\tau'}
R(\tau')e^{\mu_t(\tau')}.
\end{equation}
\end{proposition}

\begin{proof}
By Lemma~\ref{lem:bias_variance}, the $\theta$-dependent term is minimized when
\[
ZP_F(\tau)
=
R(\tau)e^{\mu_t(\tau)}
\]
for every trajectory. Full support requires this equality for all $\tau$. Summing over trajectories gives $Z=Z_t$, yielding Eq.~\ref{eq:app_target}.
\end{proof}

\begin{remark}
\label{rem:exact_target}
Proposition~\ref{prop:exact_target} is exact. The injected noise affects the optimum through
$\mathbb{E}_{\epsilon}[\log W_t]$,
while its variance contributes only the $\theta$-independent term
$\sigma_t^2$.
The result is independent of the trajectory sampling distribution as long as it has full support, including on-policy sampling. When $x_0$ changes during training, the proposition describes the population target associated with the current $x_0$ and noise level. It does not imply convergence of the parameterized training dynamics.
\end{remark}

\subsection{Low-noise endpoint: exploitation}
\label{app:low_noise}

At $\bar{\alpha}_t=1$,
\[
\tilde{x}_0[e]=\max(x_0[e],\delta),
\qquad
h_0(u)=\log\max(u,\delta).
\]
Define
\[
W_0(\tau)
=
\prod_{e\in\tau}\max(x_0[e],\delta).
\]
When $x_0[e]\ge\delta$ along $\tau$, this reduces to
$\prod_{e\in\tau}x_0[e]$.
Proposition~\ref{prop:exact_target} then gives
\begin{equation}
\label{eq:app_low}
P_F^{0}(\tau)
\propto
R(\tau)W_0(\tau).
\end{equation}

\begin{assumption}[Edge decomposability]
\label{ass:decomposable}
There exists a constant $c>0$ such that
\[
\prod_{e\in\tau}x_0[e]
=
cR(\tau)
\]
for every trajectory $\tau$.
\end{assumption}

\begin{proposition}[Low-noise endpoint]
\label{prop:low_noise}
(a) Under Assumption~\ref{ass:decomposable} and $x_0[e]\ge\delta$,
\[
P_F^{0}(\tau)\propto R(\tau)^2.
\]
(b) Without Assumption~\ref{ass:decomposable}, let
$\pi_R(\tau)\propto R(\tau)$. Then
\begin{equation}
\label{eq:app_cov}
\mathbb{E}_{P_F^{0}}[R]
-
\mathbb{E}_{\pi_R}[R]
=
\frac{
\mathrm{Cov}_{\pi_R}(R,W_0)
}{
\mathbb{E}_{\pi_R}[W_0]
}.
\end{equation}
\end{proposition}

\begin{proof}
For (a), substitute
$W_0(\tau)=cR(\tau)$
into Eq.~\ref{eq:app_low}.
For (b),
\[
P_F^{0}(\tau)
=
\frac{\pi_R(\tau)W_0(\tau)}
{\mathbb{E}_{\pi_R}[W_0]}.
\]
Therefore,
\[
\mathbb{E}_{P_F^{0}}[R]
=
\frac{\mathbb{E}_{\pi_R}[RW_0]}
{\mathbb{E}_{\pi_R}[W_0]},
\]
which gives Eq.~\ref{eq:app_cov}.
\end{proof}

Our CDF-based construction of $x_0$ does not satisfy Assumption~\ref{ass:decomposable} exactly, so the $R^2$ form is a special case rather than a general property. Equation~\ref{eq:app_cov} requires no such assumption: the low-noise target has higher expected reward than reward-proportional sampling whenever
$\mathrm{Cov}_{\pi_R}(R,W_0)>0$.
This is precisely the behavior encouraged by constructing $x_0$ from edge-reward statistics.

\subsection{High-noise endpoint: exploration}
\label{app:high_noise}

\begin{proposition}[High-noise endpoint]
\label{prop:high_noise}
At $\bar{\alpha}_t=0$,
\[
\mu_T(\tau)=c_T|\tau|,
\]
where
\[
c_T
=
\mathbb{E}_{\epsilon}
\left[
\log
\max(\mathrm{clamp}(\epsilon,0,1),\delta)
\right].
\]
Hence,
\[
P_F^{T}(\tau)\propto R(\tau)
\]
for TSP, while for CVRP,
\[
P_F^{T}(\tau)
\propto
R(\tau)e^{c_TK(\tau)}.
\]
\end{proposition}

\begin{proof}
At $\bar{\alpha}_t=0$,
$\tilde{x}_T[e]$ depends only on $\epsilon[e]$ and is independent of $x_0[e]$. Thus
$h_T(u)=c_T$
for all $u$, giving
$\mu_T(\tau)=c_T|\tau|$.

For TSP, every tour contains $n$ edges, so
$e^{c_Tn}$
is absorbed into the normalizer. For CVRP,
$|\tau|=n+K(\tau)$;
after absorbing $e^{c_Tn}$, the remaining factor is
$e^{c_TK(\tau)}$.
\end{proof}

\begin{remark}[CVRP]
\label{rem:cvrp}
For the implemented value $\delta=10^{-10}$, numerical integration gives
$c_T\approx-11.9$.
The CVRP endpoint is therefore reward-proportional among solutions with the same number of routes, while favoring fewer routes across groups with different $K(\tau)$.

On the evaluated samples ($n=500$ and $n=1000$, 32 instances each), route count and cost are positively correlated on every instance, with mean correlations of $+0.53$ and $+0.36$, respectively. This does not imply that fewer routes are always preferable or that near-optimal solutions always use the minimum number of routes. Therefore, the unconditional reward-proportional endpoint is exact for TSP only; for CVRP, the exploration benefit is evaluated empirically.
\end{remark}

\subsection{Rate between the endpoints}
\label{app:rate}

\begin{lemma}[Monotonicity and Lipschitz constant of $h_t$]
\label{lem:lipschitz}
Let $\bar{\alpha}_t<1$. For $0\le u\le v\le1$,
\begin{equation}
0
\le
h_t(v)-h_t(u)
\le
K_t(v-u),
\qquad
K_t
\le
\frac{
\sqrt{\bar{\alpha}_t}\log(1/\delta)
}{
\sqrt{2\pi(1-\bar{\alpha}_t)}
}.
\end{equation}
\end{lemma}

\begin{proof}
Let
$s_t=\sqrt{1-\bar{\alpha}_t}$ and
\[
\ell(z)
=
\log\max(\mathrm{clamp}(z,0,1),\delta).
\]
Then
\[
h_t(u)
=
\mathbb{E}_{\epsilon}
\left[
\ell(\sqrt{\bar{\alpha}_t}u+s_t\epsilon)
\right].
\]
Since $\ell$ is nondecreasing,
$h_t(v)\ge h_t(u)$.
Moreover,
$\ell'(z)=1/z$ on $(\delta,1)$ and $0$ elsewhere. Therefore,
\begin{equation}
h_t'(u)
=
\sqrt{\bar{\alpha}_t}
\int_{\delta}^{1}
\frac{1}{z}
\frac{1}{s_t}
\varphi
\left(
\frac{
z-\sqrt{\bar{\alpha}_t}u
}{
s_t
}
\right)
dz
\le
\frac{
\sqrt{\bar{\alpha}_t}\log(1/\delta)
}{
s_t\sqrt{2\pi}
},
\end{equation}
where $\varphi$ is the standard normal density and
$\varphi\le1/\sqrt{2\pi}$.
\end{proof}

For $\delta=10^{-10}$,
$\log(1/\delta)\approx23.0$,
so the constant is explicit.

\begin{theorem}[Distance to reward-proportional sampling, TSP]
\label{thm:kl_bound}
For TSP with $n$ nodes and $\bar{\alpha}_t<1$, let $P_F^{t}$ be the target in Proposition~\ref{prop:exact_target} and
$\pi_R(\tau)\propto R(\tau)$. Then
\begin{equation}
\label{eq:app_kl}
\mathrm{KL}
\big(
P_F^{t}\,\|\,\pi_R
\big)
\le
nK_t
\le
\frac{
n\sqrt{\bar{\alpha}_t}\log(1/\delta)
}{
\sqrt{2\pi(1-\bar{\alpha}_t)}
}
\xrightarrow[\bar{\alpha}_t\to0]{}
0.
\end{equation}
\end{theorem}

\begin{proof}
Let
\[
g(\tau)
=
\log\frac{P_F^{t}(\tau)}{\pi_R(\tau)}
=
\mu_t(\tau)
-
\log\frac{Z_t}{Z_R},
\qquad
Z_R=\sum_{\tau}R(\tau).
\]
Each TSP tour contains $n$ edges. By Lemma~\ref{lem:lipschitz}, each term
$h_t(x_0[e])$
varies over an interval of length at most $K_t$. Hence,
$\mu_t(\tau)$
varies by at most $nK_t$, and so does $g(\tau)$.

Because both distributions are normalized,
$\mathbb{E}_{\pi_R}[e^g]=1$,
which implies
$\inf_{\tau}g\le0\le\sup_{\tau}g$.
Therefore,
\[
\mathrm{KL}(P_F^{t}\,\|\,\pi_R)
=
\mathbb{E}_{P_F^{t}}[g]
\le
\sup_{\tau}g
\le
nK_t.
\]
\end{proof}

\textbf{Signal correlation.}
Let $\sigma_0^2$ denote the variance of $x_0[e]$ over edges. Since $\epsilon$ is independent of $x_0$,
\[
\mathrm{Cov}(x_t,x_0)
=
\sqrt{\bar{\alpha}_t}\sigma_0^2,
\]
and
\[
\mathrm{Var}(x_t)
=
\bar{\alpha}_t\sigma_0^2+1-\bar{\alpha}_t.
\]
Thus,
\begin{equation}
\rho(t)
=
\sqrt{
\frac{
\bar{\alpha}_t\sigma_0^2
}{
\bar{\alpha}_t\sigma_0^2+1-\bar{\alpha}_t
}
},
\end{equation}
which decreases monotonically from $1$ to $0$.

As $\bar{\alpha}_t\to0$,
$\rho(t)=\Theta(\sqrt{\bar{\alpha}_t})$,
so Eq.~\ref{eq:app_kl} vanishes at rate
$O(n\rho(t))$.
The dependence on $n$ also shows that the bound is mainly informative late in the schedule, when
$\sqrt{\bar{\alpha}_t}\ll1/n$.
Since $\bar{\alpha}_t$ decreases with $t$, the bound tightens throughout training. The edge-dependent reward signal is gradually weakened, and the TSP target converges to reward-proportional sampling at the endpoint. For CVRP, the same argument applies among trajectories with the same number of routes.

\subsection{Trajectories and solutions}
\label{app:multiplicity}

$P_F$ is defined over trajectories rather than solutions. For TSP, each tour has the same trajectory multiplicity---between $1$ and $2n$, depending on the convention for starting node and direction. This common factor cancels, so reward proportionality over trajectories also holds over solutions.

For CVRP, the number of trajectories representing a solution can depend on its route structure. This introduces the same qualification as the factor
$e^{c_TK(\tau)}$
in Remark~\ref{rem:cvrp}; the issue disappears when comparing solutions with the same route structure.

\section{More Experiment Detail}
\subsection{GSA Comparison Results}
\label{sec:GSA Comparison Results}
We validate the behavior that the best outcomes of larger-scale instances often appear at earlier training stages. We evaluate the model on datasets of varying sizes (200, 500, and 1000 nodes), each containing 128 test cases. The model is trained on 100-node instances and evaluated using checkpoints saved at different training epochs, with each epoch comprising 20 training steps. The results are presented in Fig. \ref{fig:origin}. As the instance size increases from 200 to 500 and then to 1000 nodes, the performance peak occurs progressively earlier during training. Moreover, the model's performance degrades significantly after reaching its peak, especially on larger-scale instances.
\begin{figure}[htbp]
  \centering

    \includegraphics[width=0.31\textwidth]{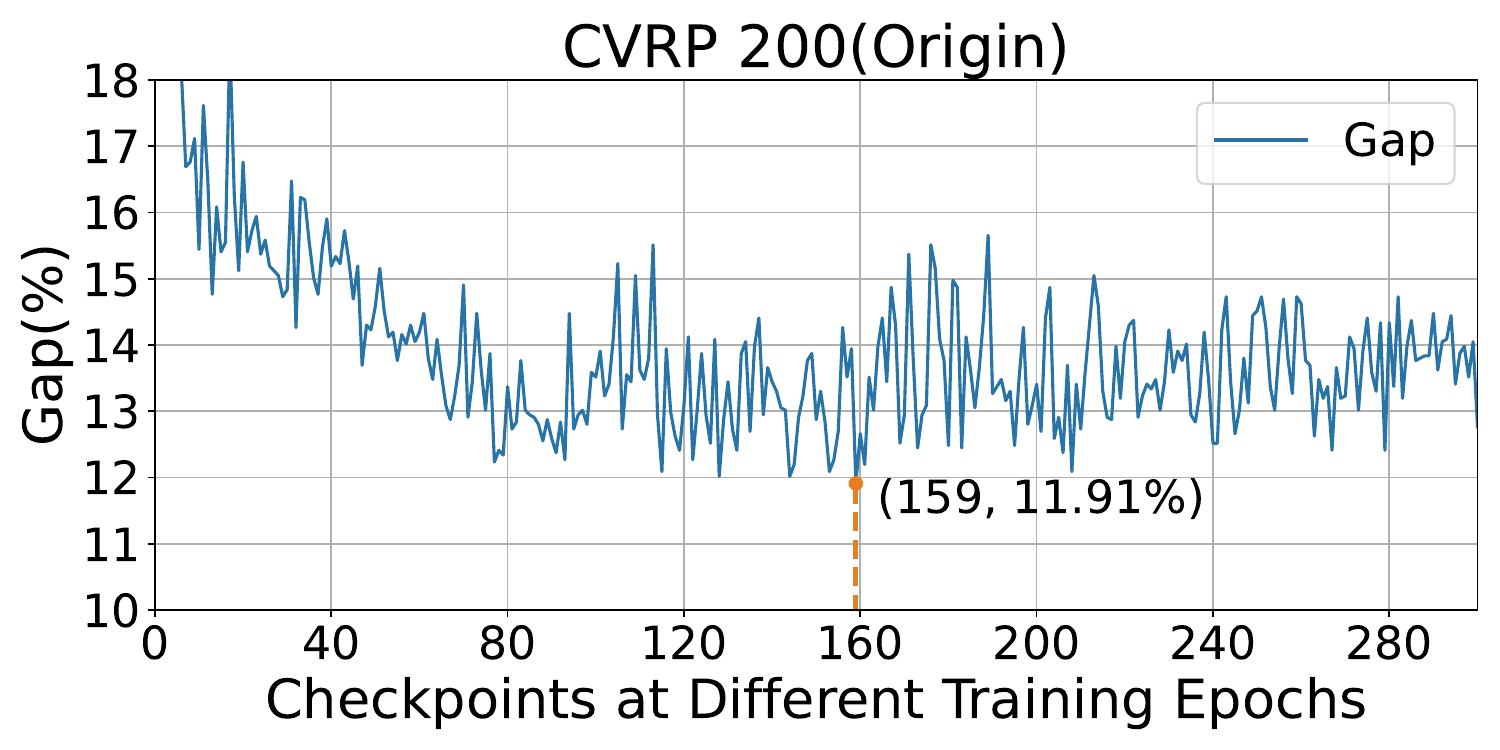}
    \includegraphics[width=0.31\textwidth]{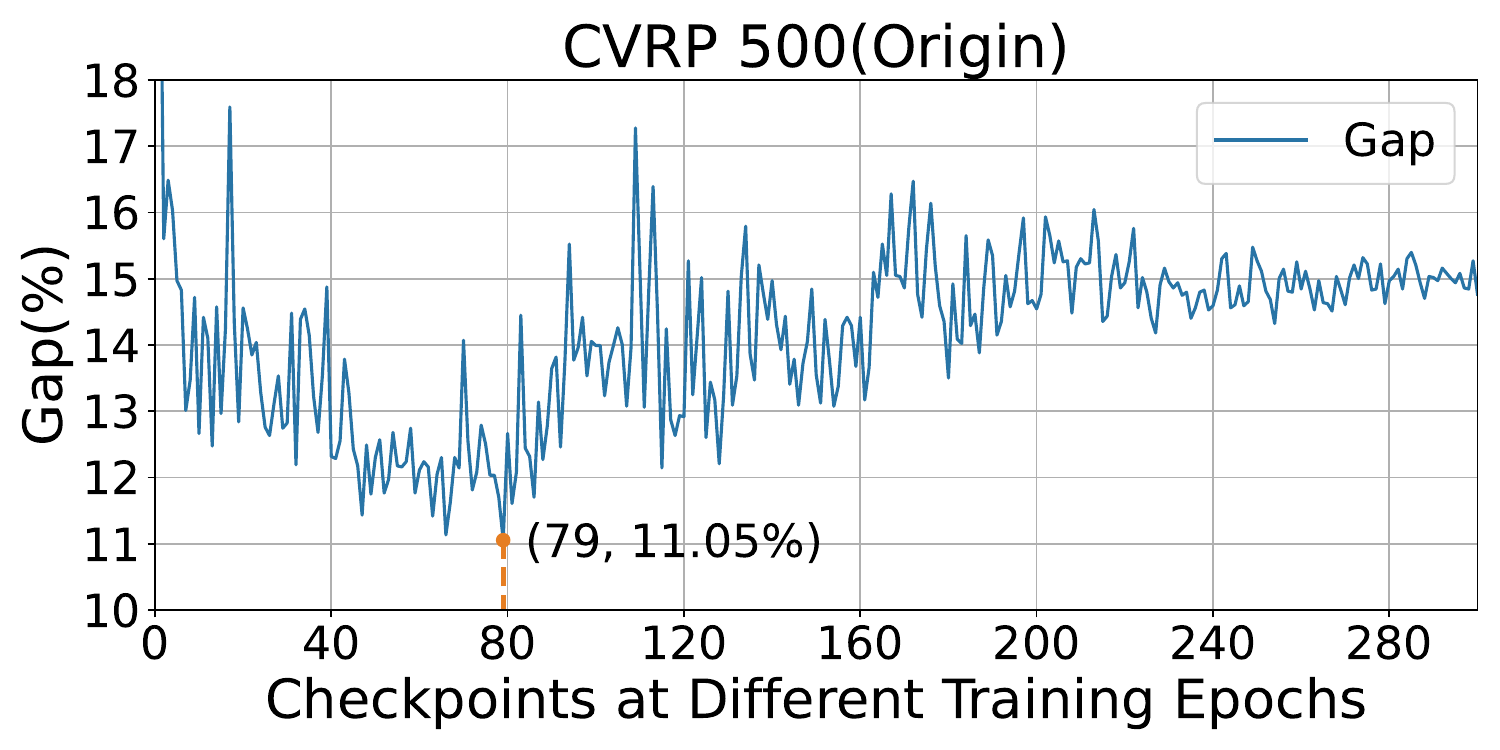}
    \includegraphics[width=0.31\textwidth]{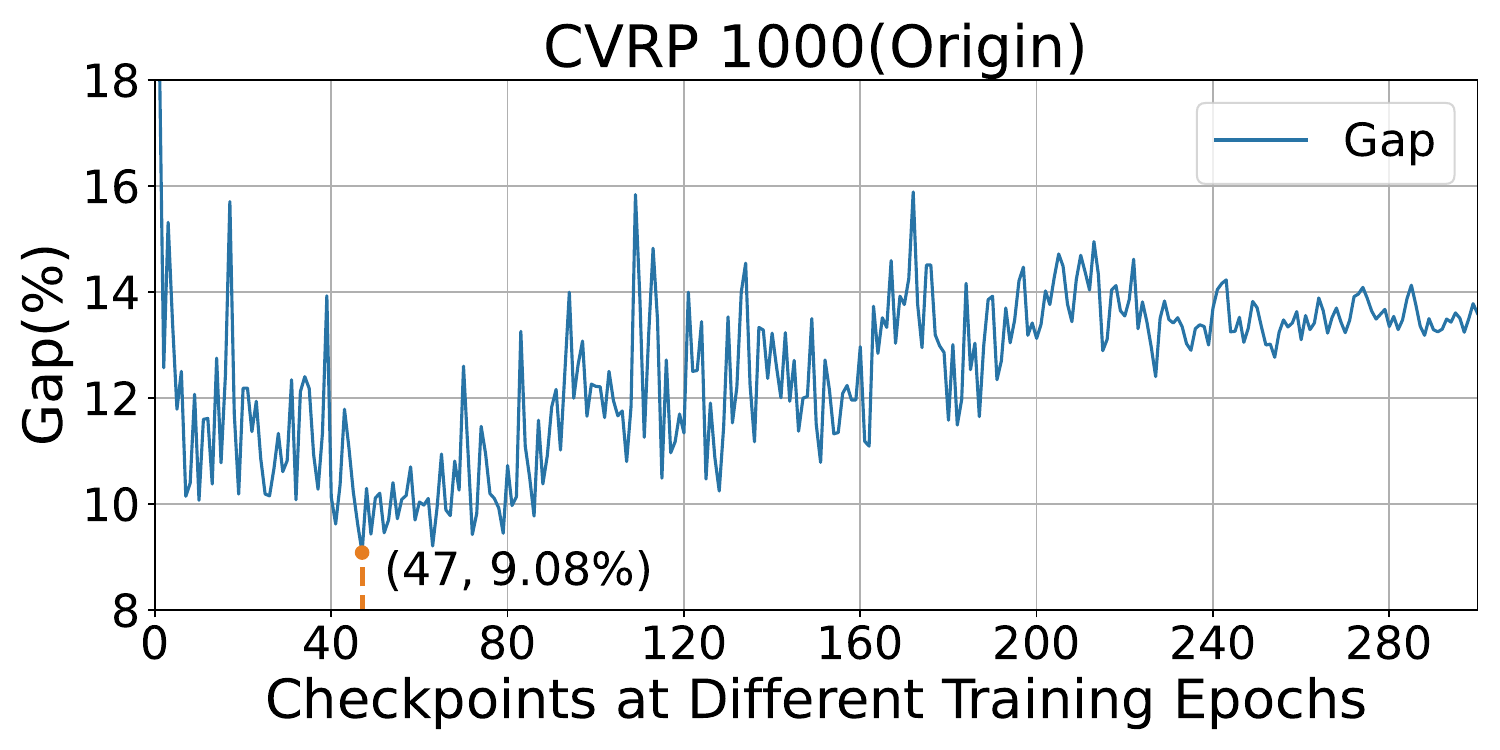}

  \caption{The original performance of DEG test on different scale instance. Gap (\%) is computed relative to LKH, and lower values indicate better results.}
  \label{fig:origin}
\end{figure}

After using GSA, as illustrated in Fig.~\ref{fig:new}, GSA successfully constrains the optimal inference checkpoints across different instance sizes into a consistent interval. Moreover, as training progresses, the model shows a steadily improving performance at later checkpoints—unlike the original decoder, which tends to degrade after early peaks. In addition, GSA achieves overall better performance compared to the baseline setup without GSA. These results indicate that GSA mitigates spatial pattern overfitting to the training instance size and enables the model to derive consistent results across varying instance scales.
\begin{figure}[htbp]
  \centering
    \includegraphics[width=0.31\textwidth]{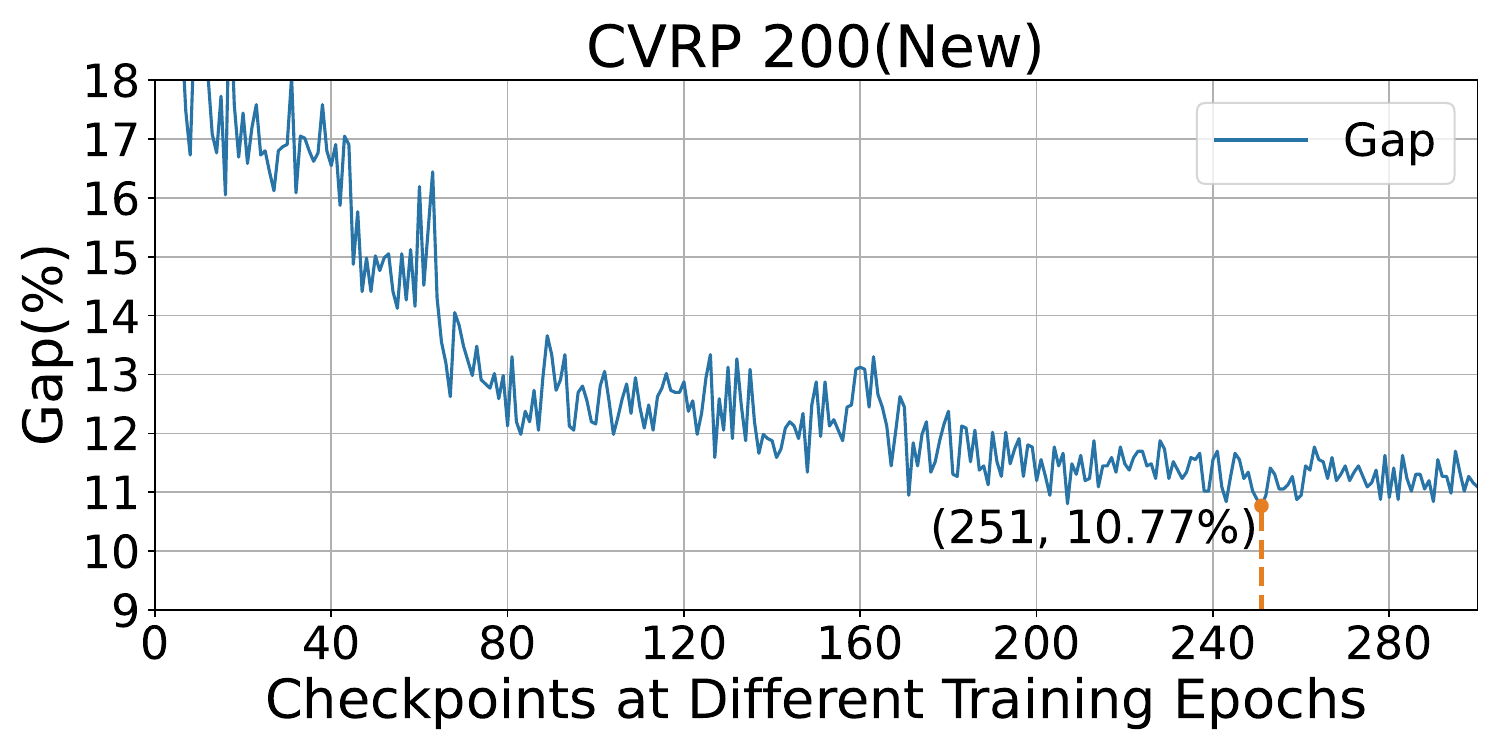}
    \includegraphics[width=0.31\textwidth]{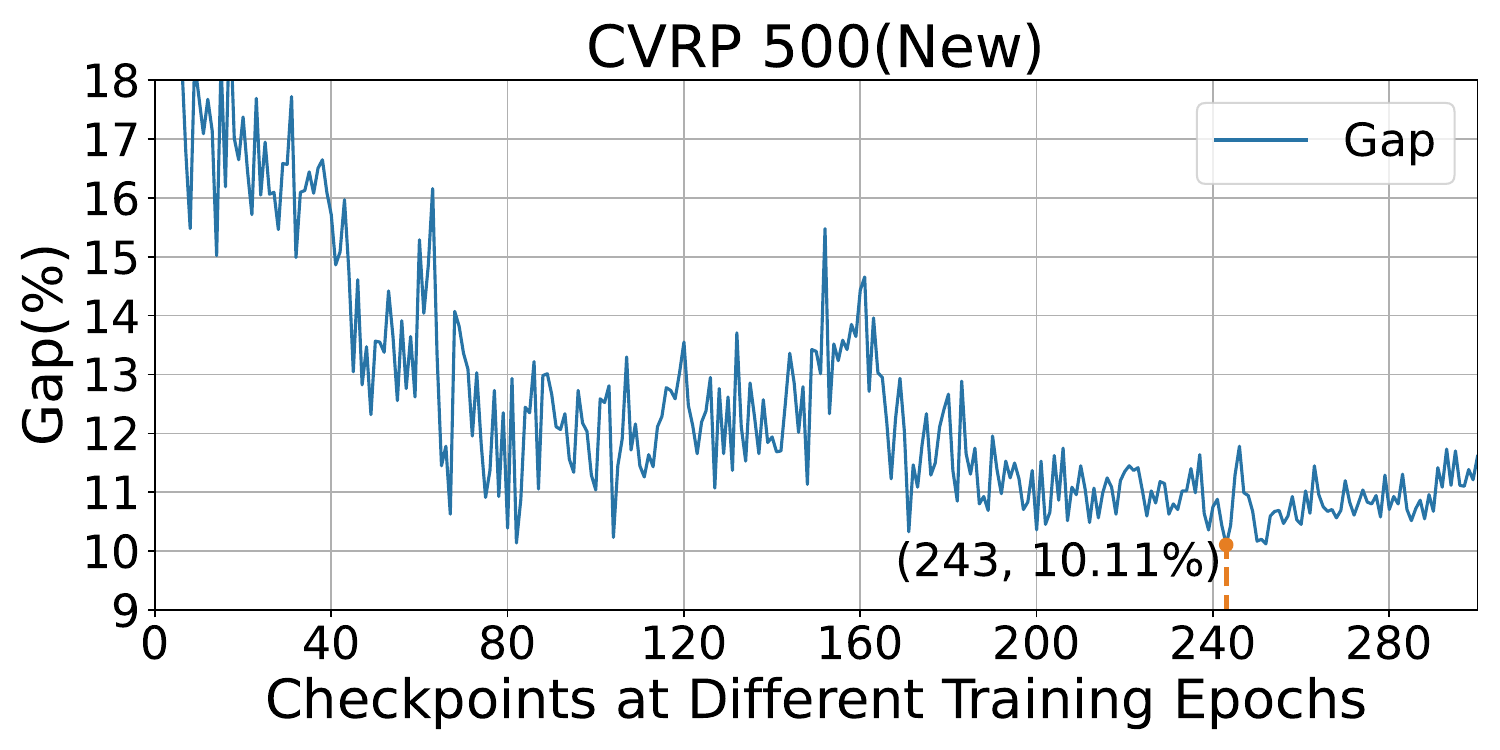}
    \includegraphics[width=0.31\textwidth]{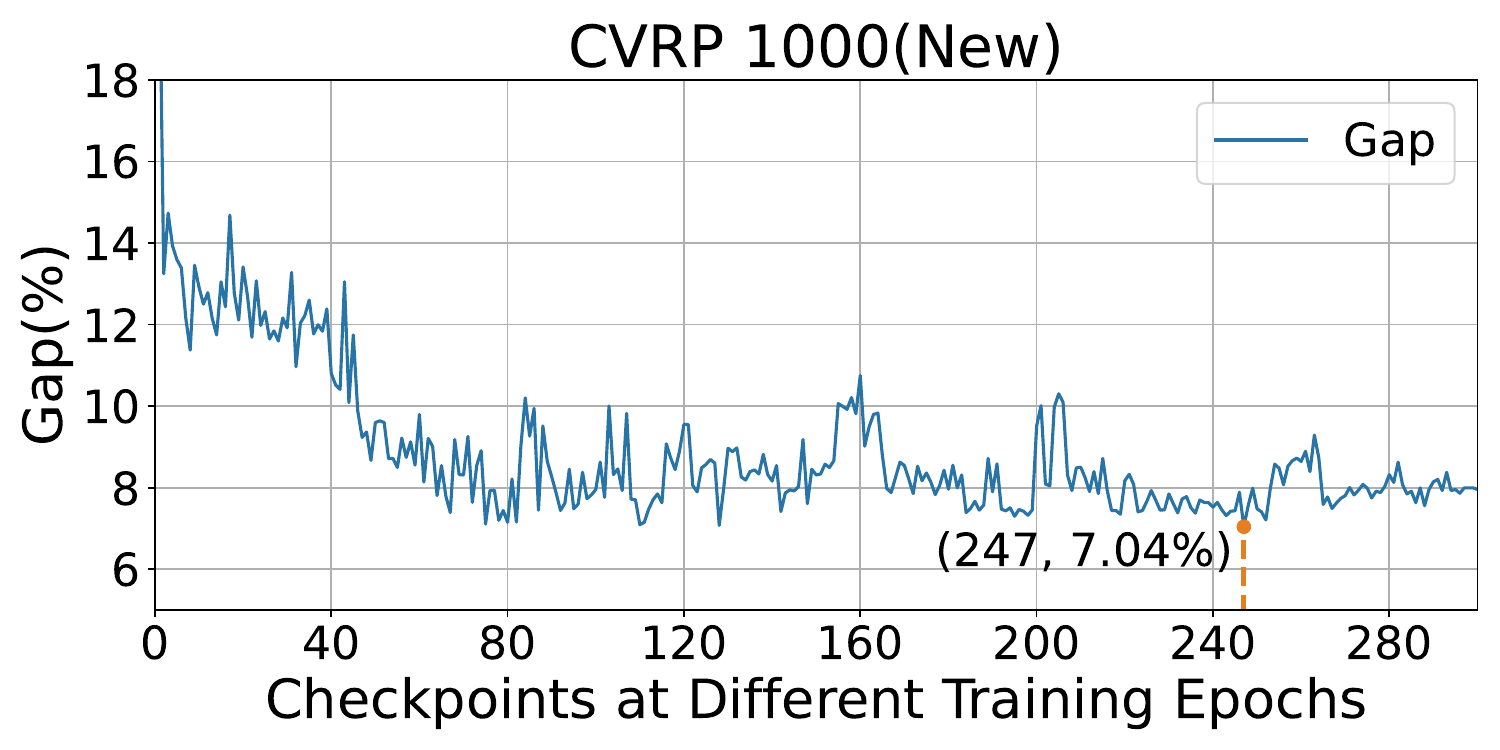}
  \caption{The performance of GSA-based DEG test on different scale instance. Gap (\%) is computed relative to LKH, and lower values indicate better results.}
  \label{fig:new}
\end{figure}

\subsection{TSP comparison on Diffusion-Based Solver with Two-Opt}

\label{Dec:Two-opt}
\begin{table*}[t]
\centering
\caption{Comparison of performance on TSP with varying node sizes with two-opt post processing. Obj. denotes the total route length and Gap (\%) is computed relative to LKH. For all metrics, lower values indicate better results.}

\resizebox{1.0\textwidth}{!}{
\begin{tabular}{l|cc|cc|cc|cc|cc|cc}
\toprule\toprule
\multirow{2}{*}{\textbf{TSP}} & \multicolumn{2}{c|}{\textbf{500}} & \multicolumn{2}{c|}{\textbf{1000}} & \multicolumn{2}{c|}{\textbf{3000}}& \multicolumn{2}{c|}{\textbf{5000}}& \multicolumn{2}{c|}{\textbf{8000}}& \multicolumn{2}{c}{\textbf{10000}} \\
& \textbf{Obj.} & \textbf{Gap(\%)} & \textbf{Obj.}  & \textbf{Gap(\%)} & \textbf{Obj.}  & \textbf{Gap(\%)}& \textbf{Obj.}  & \textbf{Gap(\%)}& \textbf{Obj.}  & \textbf{Gap(\%)}& \textbf{Obj.}  & \textbf{Gap(\%)} \\
\midrule
DIFUSCO (NeurIPS’23) & 17.32 & 4.65 & 24.22 & 4.66&43.41&7.69&--&--&--&--&--&--\\
T2T (NeurIPS’23) & 17.14 & 3.56 & 24.18 & 4.49&42.97 &6.60 &--&--&--&--&--&--\\
Fast T2T (NeurIPS’24) & 17.05 & 3.02 & 24.07 & 4.02&44.06&9.30&--&--&--&--&--&--\\
DEITSP (KDD’25)&16.98&2.60&23.86&3.15&43.10&6.92&--&--&--&--&--&--\\
DEG  & \textbf{16.90} & \textbf{2.11} & \textbf{23.67} & \textbf{2.30}&\textbf{42.22}&\textbf{4.74}&\textbf{54.61}&\textbf{5.34}&\textbf{69.46}& \textbf{5.42}&\textbf{80.05}&\textbf{5.69}\\

\bottomrule\bottomrule
\end{tabular}
}
\begin{flushleft}
\small\textit{``--'' indicates that the model was unable to produce a result for a given instance either within the 90-minute time limit or due to GPU memory constraints.}
\end{flushleft}
\label{tab:tsp_two-opt}
\end{table*}

We compare DEG with existing diffusion-based models, all enhanced with Two-Opt post-processing, and the results are summarized in Table~\ref{tab:tsp_two-opt}. Our method, DEG, achieves an objective value reduction of $2.42\%$, $2.27\%$, and $2.74\%$ on 500-, 1000-, 3000-node instance over DIFUSCO. Compared to Fast-T2T, DEG yields improvements of $0.88\%$, $1.66\%$, and $4.17\%$ on the same instance sizes. Furthermore, DEG outperforms DEITSP by $0.80\%$, $2.04\%$ on 500-, 1000-, 3000-node instances. Notably, DEG is capable of scaling to larger problem sizes, specifically 5,000, 8,000 and 10,000 nodes, whereas other diffusion-based models fail to run due to memory limitations.

\subsection{TSP and CVRP comparison in small-scale problems}
\label{sec:TSP and CVRP comparison in small-scale problems}
\begin{table*}[t]
\centering

\caption{Comparison on small-scale TSP and CVRP instances (100 and 200 Nodes). Obj. denotes the total route length; Gap (\%) is computed relative to LKH; and Time (s) indicates the average time required to solve a single instance. For all metrics, lower values indicate better results.}
\label{tab:smallscale_sidebyside}
\renewcommand{\arraystretch}{1.15}

\begin{minipage}{0.55\textwidth}
\centering
\vspace{1mm}
\resizebox{\linewidth}{!}{
\begin{tabular}{l|ccc|ccc}
\toprule
\multirow{1}{*}{\textbf{TSP}} & \multicolumn{3}{c|}{\textbf{100}} & \multicolumn{3}{c}{\textbf{200}} \\
 & \textbf{Obj.} & \textbf{Time(s)} & \textbf{Gap(\%)} 
             & \textbf{Obj.} & \textbf{Time(s)} & \textbf{Gap(\%)} \\
\midrule
LKH                         & 7.76     & 1.50  & -        & 10.62     & 9.80  & -        \\
GFACS                       & 8.80     & 0.53  & 13.54    & 13.20     & 1.15  & 24.29    \\
AGFN                        & 8.49     & 0.05  & 9.55     & 11.82     & 0.08  & 11.30    \\
DEG                         & \textbf{8.40}     & 0.06  & \textbf{8.25}     & \textbf{11.65}     & 0.10  & \textbf{9.42}     \\
DIFUSCO (w/ LS)             & 7.88     & 0.42  & 1.68     & 10.80     & 1.53  & 1.70     \\
T2T (w/ LS)                 & 7.80     & 0.37  & 0.65     & 10.80     & 1.15  & 1.72     \\
DEITSP (w/ LS)              & \textbf{7.78} & 0.08  & \textbf{0.39}
                                         & 10.76     & 1.20  & 1.32     \\
Fast T2T (w/ LS)            & 7.82     & 0.11  & 0.90     & 10.78     & 1.19  & 1.51     \\
DEG (w/ LS)                 & 7.79     & 0.07  & 0.52     & \textbf{10.76} & 1.15 & \textbf{1.32} \\
\bottomrule
\end{tabular}}
\end{minipage}
\hfill
\begin{minipage}{0.43\textwidth}
\centering
\vspace{1mm}
\resizebox{\linewidth}{!}{
\begin{tabular}{l|ccc|ccc}
\toprule
\multirow{1}{*}{\textbf{CVRP}} & \multicolumn{3}{c|}{\textbf{100}} & \multicolumn{3}{c}{\textbf{200}} \\
 & \textbf{Obj.} & \textbf{Time(s)} & \textbf{Gap(\%)} 
             & \textbf{Obj.} & \textbf{Time(s)} & \textbf{Gap(\%)} \\
\midrule
LKH    & 15.57     & 18.80   & -        & 28.04     & 50.83   & -         \\
GFACS  & 19.26     & 1.90    & 23.70    & 35.36     & 3.14    & 26.11     \\
AGFN   & 17.78     & 0.06    & 14.19    & 31.60     & 0.16    & 12.70     \\
DEG    & \textbf{17.07} & 0.07 & \textbf{9.63}
                 & \textbf{31.00} & 0.18 & \textbf{10.56} \\
\bottomrule
\end{tabular}}
\end{minipage}

\end{table*}

To evaluate the performance of DEG on small-scale problems, we additionally conduct experiments on TSP and CVRP instances of 100  and 200 nodes. The comparisons include both diffusion-based and GFlowNet-based baselines; for CVRP, only GFlowNet-based baselines are reported because diffusion-based solvers do not directly apply to this setting. The results, summarized in Table~\ref{tab:smallscale_sidebyside}, show that DEG consistently surpasses all diffusion-based and GFlowNet-based methods on TSP200, CVRP100, and CVRP200, and achieves competitive or superior performance on TSP100. These findings demonstrate that DEG also maintains strong optimization capability on problems of relatively small scales. 
\begin{table*}[t]
\centering
\renewcommand{\arraystretch}{1.15}

\caption{Training on different scales for TSP and testing without and with GSA. Obj. denotes the total route length, and Gap (\%) is computed relative to LKH. For both metrics, lower is better.}

\resizebox{\textwidth}{!}{
\begin{tabular}{lcccccccccccc}
\toprule
& \multicolumn{2}{c}{\textbf{500}} 
& \multicolumn{2}{c}{\textbf{1000}} 
& \multicolumn{2}{c}{\textbf{3000}} 
& \multicolumn{2}{c}{\textbf{5000}} 
& \multicolumn{2}{c}{\textbf{8000}} 
& \multicolumn{2}{c}{\textbf{10000}}\\

& Obj. & Gap(\%) 
& Obj. & Gap(\%) 
& Obj. & Gap(\%) 
& Obj. & Gap(\%) 
& Obj. & Gap(\%) 
& Obj. & Gap(\%) \\
\midrule
100(w/o GSA)  & 18.54 & 12.02 & 26.74 & 15.56 & 48.45 & 20.19 & 64.97 & 25.33 & -- & -- & -- & -- \\
200(w/o GSA)  & 18.44 & 11.42 & 26.50 & 14.52 & 47.59 & 18.06 & 64.16 & 23.77 & -- & -- & -- & -- \\
500(w/o GSA)  & \textbf{18.05} & \textbf{9.06} & 26.42 & 14.17 & 47.19 & 17.07 & \textbf{61.55} & \textbf{18.73} & -- & -- & -- & -- \\
1000(w/o GSA) & 18.41 & 11.24 & \textbf{25.87} & \textbf{11.80} & \textbf{46.32} & \textbf{14.91} & 63.14 & 21.80 & -- & -- & -- & -- \\
\midrule
100(w/ GSA)  & 18.47 & 11.60 & 26.46 & 14.35 & 47.82 & 18.63 & 63.64 & 22.57 & 83.13 & 26.16 & 94.67 & 24.94 \\
200(w/ GSA)  & 18.33 & 10.76 & 26.34 & 13.83 & 46.91 & 16.37 & 62.16 & 19.91 & 82.19 & 24.74 & 91.09 & 20.20 \\
500(w/ GSA)  & \textbf{18.08} & \textbf{9.24} & 25.97 & 12.23 & 46.05 & \textbf{14.24} & \textbf{61.11} & \textbf{17.88} & \textbf{78.16} & \textbf{18.62} & \textbf{87.44} & \textbf{15.39} \\
1000(w/ GSA) & 18.28 & 10.45 & \textbf{25.87} & \textbf{11.80} & \textbf{45.09} & 14.86 & 62.31 & 20.20 & 79.06 & 19.99 & 89.58 & 18.21 \\

\bottomrule
\end{tabular}

}
\begin{flushleft}
\small\textit{``--'' indicates that the model was unable to produce a result for a given instance either within the 90-minute time limit or due to GPU memory constraints.}
\end{flushleft}
\label{tab:g_train_tsp}
\end{table*}

\begin{table*}[ht]
\centering
\renewcommand{\arraystretch}{1.15}

\caption{Training on different scales for CVRP and testing without and with GSA. Obj. denotes the total route length, and Gap (\%) is computed relative to LKH. For both metrics, lower values indicate better results.}

\resizebox{\textwidth}{!}{
\begin{tabular}{lcccccccccccc}
\toprule
& \multicolumn{2}{c}{\textbf{500}} 
& \multicolumn{2}{c}{\textbf{1000}} 
& \multicolumn{2}{c}{\textbf{3000}} 
& \multicolumn{2}{c}{\textbf{5000}} 
& \multicolumn{2}{c}{\textbf{8000}} 
& \multicolumn{2}{c}{\textbf{10000}} \\

& Obj. & Gap(\%) 
& Obj. & Gap(\%) 
& Obj. & Gap(\%) 
& Obj. & Gap(\%) 
& Obj. & Gap(\%) 
& Obj. & Gap(\%) \\
\midrule
100 (w/o GSA)  & 69.76 & 10.48 & 129.67 & 8.49 & 362.11 & -4.77 & 596.43 & -7.24 & 907.65 & -- & 1125.04 & -- \\
200 (w/o GSA)  & 67.75 & 7.30  & 127.70 & 6.84 & 362.43 & -4.69 & 580.73 & -9.68 & 987.01 & -- & 1094.82 & -- \\
500 (w/o GSA)  & \textbf{66.76} & \textbf{5.73} & 124.70 & 4.34 & 360.80 & -5.12 & 577.24 & -10.22 & 894.98 & -- & 1072.27 & -- \\
1000(w/o GSA)  & 69.68 & 10.36 & \textbf{123.13} & \textbf{3.02} & \textbf{360.42} & \textbf{-5.21} & \textbf{568.73} & \textbf{-11.54} & \textbf{885.18} & -- & \textbf{1060.53} & -- \\
\midrule
100 (w/ GSA)   & 70.86 & 12.23 & 132.65 & 10.74 & 368.99 & -2.96 & --     & --     & --     & --     & --     & -- \\
200 (w/ GSA)   & 68.36 & 8.37  & 129.50 & 8.35  & 368.85 & -3.00 & --     & --     & --     & --     & --     & -- \\
500 (w/ GSA)   & \textbf{66.76} & \textbf{5.73} & 125.65 & 5.13 & 367.13 & -3.45 & --     & --     & --     & --     & --     & -- \\
1000(w/ GSA)   & 70.54 & 11.72 & \textbf{123.13} & \textbf{3.02} & \textbf{366.19} & \textbf{-3.70} & --     & --     & --     & --     & --     & -- \\
\bottomrule
\end{tabular}
}
\begin{flushleft}
\small\textit{``--'' indicates that the model was unable to produce a result for a given instance either within the 90-minute time limit or due to GPU memory constraints.}
\end{flushleft}
\label{tab:g_train_cvrp}
\end{table*}

\subsection{Multiple Scale training}
\label{sec:Multiple Scale training}
We conducted additional experiments where DEG is trained on multiple node sizes. Specifically, for both TSP and CVRP, we trained models on instances of size 100, 200, 500, and 1000, and evaluated them on larger test sets with 500, 1000, 3000, 5000, 8000, and 10000 nodes. For each training scale, we report results with and without GSA rescaling. 

The results (provided in Table~\ref{tab:g_train_tsp} and Table~\ref{tab:g_train_cvrp}) show that DEG maintains strong performance even without GSA, while GSA further strengthens generalization on very large instances. These findings confirm that the scalability of DEG is not dependent on training exclusively at the 100-node scale, and that the model possesses substantial inherent scale robustness.

\subsection{Diversity Test} 

As shown in Table~\ref{tab:diversity}, we evaluate diversity on both TSP and CVRP by comparing the trajectories generated by (i) a reinforcement-learning baseline, (ii) a vanilla GFlowNet implementation (the backbone of our method), and (iii) our proposed DEG model. The diversity metric~\citep{prins2009two} is defined as the average number of edges in one trajectory that do not appear in another, a standard measure of route-level diversity in VRP. In this experiment, all settings are kept identical across models, except for the training paradigm. The results show that DEG consistently achieves higher diversity than standard GFlowNets, demonstrating that our diffusion-enhanced design effectively enriches the exploration space. 
\begin{table*}[t]
\centering

\caption{Diversity comparison on TSP and CVRP,  higher values indicate better results.}
\label{tab:diversity}
\begin{minipage}{0.48\textwidth}
\centering
\resizebox{\textwidth}{!}{
\begin{tabular}{lccccc}
\toprule
\textbf{TSP} & \textbf{100} & \textbf{500} & \textbf{1000} & \textbf{3000} & \textbf{5000} \\
\midrule
RL       & 17.60 & 96.07 & 217.20 & 819.34 & 1516.64 \\
GFlowNet & 20.36 & 126.95 & 294.30 & 1074.53 & 1895.23 \\
DEG      & \textbf{21.79} & \textbf{128.97} & \textbf{302.63} & \textbf{1145.79} & \textbf{2082.36} \\
\bottomrule
\end{tabular}}
\end{minipage}
\hfill
\begin{minipage}{0.48\textwidth}
\centering
\resizebox{\textwidth}{!}{
\begin{tabular}{lccccc}
\toprule
\textbf{CVRP} & \textbf{100} & \textbf{500} & \textbf{1000} & \textbf{3000} & \textbf{5000} \\
\midrule
RL       & 16.57 & 181.36 & 400.71 & 1210.16 & 2372.48 \\
GFlowNet & 19.33 & 212.27 & 443.73 & 1526.56 & 2705.91 \\
DEG      & \textbf{20.54} & \textbf{217.02} & \textbf{471.68} & \textbf{1648.82} & \textbf{2909.12} \\
\bottomrule
\end{tabular}}
\end{minipage}

\end{table*}

\subsection{Learning Efficiency} 
\begin{table*}[t]
\centering

\caption{Comparison of training efficiency on TSP and CVRP. Obj. denotes the total route length, and Gap (\%) is computed relative to LKH. For both metrics, lower values indicate better results.}
\label{tab:training_efficiency}
\renewcommand{\arraystretch}{1.15}

\begin{minipage}{0.49\textwidth}
\centering
\resizebox{\linewidth}{!}{
\begin{tabular}{l|cc|cc|cc}
\toprule
\multirow{1}{*}{TSP} & \multicolumn{2}{c|}{\textbf{200}} & \multicolumn{2}{c|}{\textbf{500}} & \multicolumn{2}{c}{\textbf{1000}} \\
& Obj. & Gap(\%) & Obj. & Gap(\%) & Obj. & Gap(\%) \\
\midrule
RL & 12.12 & 14.12 & 19.28 & 16.50 & 27.82 & 20.22 \\
GFlowNet      & 12.08 & 13.75 & 19.02 & 14.92 & 27.29 & 17.93 \\
PO                     & 11.97 & 12.71 & 18.93 & 14.38 & 27.11 & 17.16 \\
DEG                    & \textbf{11.91} & \textbf{12.15} & \textbf{18.79} & \textbf{13.53} & \textbf{26.80} & \textbf{15.82} \\
\bottomrule
\end{tabular}}
\end{minipage}
\hfill
\begin{minipage}{0.49\textwidth}
\centering
\vspace{1mm}
\resizebox{\linewidth}{!}{
\begin{tabular}{l|cc|cc|cc}
\toprule
\multirow{1}{*}{CVRP} & \multicolumn{2}{c|}{\textbf{200}} & \multicolumn{2}{c|}{\textbf{500}} & \multicolumn{2}{c}{\textbf{1000}} \\
& Obj. & Gap(\%) & Obj. & Gap(\%) & Obj. & Gap(\%) \\
\midrule
RL & 32.94 & 17.48 & 72.50 & 14.82 & 133.41 & 11.62 \\
GFlowNet      & 32.03 & 14.23 & 71.49 & 13.22 & 132.47 & 10.84 \\
PO                     & 32.13 & 14.59 & 71.77 & 13.67 & 133.39 & 11.60 \\
DEG                    & \textbf{31.90} & \textbf{13.77} & \textbf{70.15} & \textbf{11.10} & \textbf{131.23} & \textbf{9.80} \\
\bottomrule
\end{tabular}}
\end{minipage}

\end{table*}

We conducted additional experiments on a fixed problem scale and distribution to assess learning efficiency under controlled settings. Specifically, we used 10{,}000 uniformly sampled TSP and CVRP instances with size $n=200$ and trained all models with identical configurations, differing only in the training paradigm (RL, standard GFlowNet, PO~\citep{panpreference}, and our DEG). The trained models were then evaluated on TSP/CVRP test instances at $n=200$, $n=500$, and $n=1000$. The results, reported in Appendix Table~\ref{tab:training_efficiency}, show that DEG exhibits strong optimization performance and superior learning efficiency compared to these baselines under matched training conditions.

\subsection{Applicability to Other Combinatorial Optimization Problems}
\label{sec:Application to Other Combinatorial Optimization Problems}
\begin{table*}[ht]
\centering

\caption{Results on SMTWTP and BPP benchmarks. For SMTWTP, Obj. denotes total weighted tardiness (lower is better); for BPP, Obj. denotes space utilization rate (higher is better). Gap (\%) is computed relative to LKH, where lower values indicate better performance.}
\label{tab:general_CO}
\renewcommand{\arraystretch}{1.15}
\begin{minipage}{0.49\textwidth}
\centering
\resizebox{\linewidth}{!}{
\begin{tabular}{l|cc|cc|cc}
\toprule
 SMTWTP& \multicolumn{2}{c|}{\textbf{100}} 
 & \multicolumn{2}{c|}{\textbf{200}} 
 & \multicolumn{2}{c}{\textbf{500}}\\
 & \textbf{Obj.}$\downarrow$ & \textbf{Time(s)} 
 & \textbf{Obj.}$\downarrow$ & \textbf{Time(s)} 
 & \textbf{Obj.}$\downarrow$ & \textbf{Time(s)} \\
\midrule
ACO     & 361.22 & 0.82 & 1608.26 & 1.70 & 11015.39 & 2.84 \\
DeepACO & 223.36 & 0.93 & 905.19  & 1.81 & 1567.20  & 3.56 \\
GFACS   & 40.67  & 1.04 & 119.32  & 2.25 & 559.31   & 3.68 \\
DEG     & \textbf{37.70} & \textbf{0.14}
        & \textbf{103.13} & \textbf{0.34}
        & \textbf{510.67} & \textbf{0.98} \\
\bottomrule
\end{tabular}}
\end{minipage}
\hfill
\begin{minipage}{0.49\textwidth}
\centering
\resizebox{\linewidth}{!}{
\begin{tabular}{l|cc|cc|cc}
\toprule
 BPP& \multicolumn{2}{c|}{\textbf{120}} 
 & \multicolumn{2}{c|}{\textbf{250}} 
 & \multicolumn{2}{c}{\textbf{500}}\\
 & \textbf{Obj.}$\uparrow$ & \textbf{Time(s)} 
 & \textbf{Obj.}$\uparrow$ & \textbf{Time(s)} 
 & \textbf{Obj.}$\uparrow$ & \textbf{Time(s)} \\
\midrule
ACO     & 90.65 & 2.95 & 90.36 & 5.85 & 90.17 & 10.00 \\
DeepACO & 93.46 & 2.37 & 93.10 & 5.50 & 93.08 & 10.75 \\
GFACS   & 94.68 & 2.38 & 94.03 & 5.56 & 93.65 & 10.77 \\
DEG     & \textbf{96.94} & \textbf{0.27}
        & \textbf{96.30} & \textbf{0.58}
        & \textbf{95.37} & \textbf{1.38} \\
\bottomrule
\end{tabular}}
\end{minipage}

\end{table*}

To further assess the generalization ability of our proposed DEG beyond routing problems, we extended our experiments to additional combinatorial optimization domains. Specifically, we considered: (1) SMTWTP (Single-Machine Total Weighted Tardiness Problem), a classical single-machine scheduling problem where the objective is to minimize the total weighted tardiness of jobs processed on one machine. (2) BPP (Bin Packing Problem), a packing optimization problem that aims to minimize the number of bins required to accommodate items of varying sizes under capacity constraints, thereby improving space utilization. In Appendix Table~\ref{tab:general_CO}, regarding SMTWTP, the instance sizes (100, 200, and 500) correspond to the number of jobs, with Obj. denoting total weighted tardiness. Regarding BPP, the instance sizes (120, 250, and 500) correspond to the number of items, with Obj. (\%) denoting space utilization rate, i.e., the proportion of total bin capacity effectively used by the packed items. As suggested by those results, the proposed DEG model outperforms both DeepACO~\citep{ye2023deepaco} and GFACS~\cite{kimant} on these tasks, demonstrating its potential as a general neural combinatorial optimization framework beyond routing.

\section{Detailed Introduction of GFlowNet}
The Generative Flow Network (GFlowNet) is probabilistic generative model that treats the process of sampling from a target distribution as a sequential decision-making problem~\citep{bengio2021flow,bengio2023gflownet}. 
Formally, a GFlowNet operates on a fully-observed, deterministic Markov Decision Process (MDP) with a set of states $\mathcal{S}$, a set of actions $\mathcal{A} \subseteq \mathcal{S} \times \mathcal{S}$, and a designated initial state $s_{0}$. 
Certain states are terminal and have no outgoing actions; we denote the set of terminal states by $\mathcal{X}$. 
Each terminal state $x \in \mathcal{X}$ is reachable from $s_{0}$ through at least one (not necessarily unique) sequence of actions.

A complete trajectory is defined as $\tau = (s_{0} \to s_{1} \to \cdots \to s_{n})$, where $s_{n} \in \mathcal{X}$ and every transition satisfies $(s_{i}, s_{i+1}) \in \mathcal{A}$. 
A \emph{forward policy} $P_{F}(s' \mid s)$ defines the probability distribution over next states from $s$. 
It induces a distribution over trajectories:
\[
    P_{F}(\tau) = P_{F}(s_{0} \to s_{1} \to \cdots \to s_{n}) 
    = \prod_{i=0}^{n-1} P_{F}(s_{i+1} \mid s_{i}).
\]

Each terminal state $x \in \mathcal{X}$ is associated with a reward $R(x) > 0$, interpreted as an unnormalized probability mass. 
The learning objective of a GFlowNet is to produce a policy $P_{F}$ whose induced marginal distribution over terminal states satisfies
\[
    P_{F}(x) \propto R(x).
\]
In other words, a GFlowNet learns to sample terminal objects with probability proportional to their reward.

\paragraph{Trajectory Balance (TB).}
The trajectory balance objective~\citep{malkin2022trajectory} introduces a scalar parameter $Z_{\theta}$ (an estimate of the partition function) and a backward policy $P_{B}$. 
For a complete trajectory $\tau = (s_{0} \to \cdots \to s_{n} = x)$, the TB loss is
\[
    \ell_{\mathrm{TB}}(\tau; \theta) 
    = \left(
        \log \frac{
            Z_{\theta} 
            \prod_{i=0}^{n-1} P_{F}(s_{i+1}\mid s_{i};\theta)
        }{
            R(x)
            \prod_{i=0}^{n-1} P_{B}(s_{i}\mid s_{i+1};\theta)
        }
    \right)^{2}.
\]
If the TB loss is driven to zero for all trajectories, the forward policy satisfies the desired reward-proportional condition. 
In practice, GFlowNets predict the log-probabilities of both forward and backward transitions, as well as the log-flow $F(s)$ and the log-partition constant $\log Z$.

\paragraph{Training and Exploration.}
TB losses depend on transitions or full trajectories, but do not prescribe how trajectories should be chosen. 
A common approach is on-policy training, where trajectories are sampled based on the current forward policy $P_{F}$ (which assigns probabilities to edges) and backward policy $P_{B}$ (which determines either exploration or exploitation during training), and the TB objective is minimized through gradient descent.

\paragraph{Conditional GFlowNets.}
In many applications, including combinatorial optimization, the reward and admissible actions depend on an input instance.
Conditional GFlowNets extend the formulation so that $P_{F}$ and $P_{B}$ are conditioned on input features, enabling generalization to unseen instances and supporting tasks where solutions correspond to variable-sized or structured objects.

\section{Broad Impact and Safeguards.}
Diffusion-enhanced GFlowNet for vehicle routing problems has the potential to improve the efficiency of logistics and transportation systems, reducing operational costs, fuel consumption, and environmental impact. However, applying such approaches in real-world settings requires careful testing and validation. To support reliable use, we emphasize thorough evaluation across diverse scenarios, clearly document model assumptions and limitations, and recommend combining learning-based methods with classical optimization or constraint-checking mechanisms when appropriate. Domain-specific validation is strongly encouraged to ensure responsible deployment.
\end{document}